\documentclass[conference]{IEEEtran}
\IEEEoverridecommandlockouts
\usepackage[square, numbers]{natbib}
\usepackage{amsmath,amssymb,amsfonts,amsthm}
\usepackage{graphicx}
\usepackage{xcolor}
\def\BibTeX{{\rm B\kern-.05em{\sc i\kern-.025em b}\kern-.08em
    T\kern-.1667em\lower.7ex\hbox{E}\kern-.125emX}}

\usepackage[american]{babel}
\usepackage[utf8]{inputenc} % allow utf-8 input
\usepackage[T1]{fontenc}    % use 8-bit T1 fonts
\usepackage{url}            % simple URL typesetting
\usepackage{booktabs}       % professional-quality tables
\usepackage{amsfonts}       % blackboard math symbols
\usepackage{nicefrac}       % compact symbols for 1/2, etc.
\usepackage{microtype}      % microtypography

\usepackage{textcomp}
\usepackage{enumitem}
\usepackage{subcaption}

\usepackage{xspace}
\usepackage{algorithm}
\usepackage{algorithmic}
\usepackage{bm}
\usepackage{booktabs}
\usepackage{multirow}

\newcommand{\cora}{CORA\xspace}
\newcommand{\citeseer}{CiteSeer\xspace}

\newcommand{\pubmed}{PubMed\xspace}
\newcommand{\magcs}{Coauthor CS\xspace}
\newcommand{\magph}{Coauthor Physics\xspace}
\newcommand{\amcomp}{Amazon Computers\xspace}
\newcommand{\amphoto}{Amazon Photo\xspace}

\newcommand{\BF}[1]{\textbf{#1}}
\newcommand{\PM}[1]{\small{\(\pm{#1}\)}}

\newtheorem{theorem}{Theorem}

\newtheorem{assumption}{Assumption}

\begin{document}

\title{Self-Enhanced GNN: Improving Graph Neural Networks Using Model Outputs
% {\footnotesize \textsuperscript{*}Note: Sub-titles are not captured in Xplore and should not be used}
\thanks{This work was partially supported by GRF 14208318 from the RGC of HKSAR and the National Natural Science Foundation of China (NSFC) (Grant No. 61672552).}
}

\makeatletter
\newcommand{\linebreakand}{%
  \end{@IEEEauthorhalign}
  \hfill\mbox{}\par
  \mbox{}\hfill\begin{@IEEEauthorhalign}
}
\makeatother

\author{
\IEEEauthorblockN{Han Yang\IEEEauthorrefmark{1}, 
% \IEEEauthorblockA{\textit{The Chinese University of Hong Kong} \\
% hyang@cse.cuhk.edu.hk}
Xiao Yan\IEEEauthorrefmark{2},
% \IEEEauthorblockA{\textit{Southern University of Science and Technology} \\
% xyan@sustech.edu.cn}
Xinyan Dai\IEEEauthorrefmark{1},
% \IEEEauthorblockA{\textit{The Chinese University of Hong Kong} \\
% xydai@cse.cuhk.edu.hk}
Yongqiang Chen\IEEEauthorrefmark{1}, 
James Cheng\IEEEauthorrefmark{1}}
\IEEEauthorblockA{\IEEEauthorrefmark{1}\textit{The Chinese University of Hong Kong}\\
\IEEEauthorrefmark{2}\textit{Southern University of Science and Technology}\\
\{hyang, xydai, yqchen, jcheng\}@cse.cuhk.edu.hk,  yanx@sustech.edu.cn
}
\thanks{\IEEEauthorrefmark{2}Corresponding author.}
}
% \author{
% \IEEEauthorblockN{Han Yang}
% \IEEEauthorblockA{\textit{The Chinese University of Hong Kong} \\
% hyang@cse.cuhk.edu.hk}
% \and
% \IEEEauthorblockN{Xiao Yan}
% \IEEEauthorblockA{\textit{Southern University of Science and Technology} \\
% xyan@sustech.edu.cn}
% \linebreakand
% \IEEEauthorblockN{Xinyan Dai}
% \IEEEauthorblockA{\textit{The Chinese University of Hong Kong} \\
% xydai@cse.cuhk.edu.hk}
% \and
% \IEEEauthorblockN{Yongqiang Chen}
% \IEEEauthorblockA{\textit{The Chinese University of Hong Kong} \\
% yqchen@cse.cuhk.edu.hk}
% \and
% \IEEEauthorblockN{James Cheng}
% \IEEEauthorblockA{\textit{The Chinese University of Hong Kong} \\
% jcheng@cse.cuhk.edu.hk}
% }

\maketitle

\begin{abstract}
    Graph neural networks (GNNs) have received much attention recently because of their excellent performance on graph-based tasks. However, existing research on GNNs  focuses on designing more effective models without considering much about the quality of the input data. In this paper, we propose \textit{self-enhanced GNN} (SEG), which improves the quality of the input data using the outputs of existing GNN models for better performance on semi-supervised node classification. As graph data consist of both topology and node labels, we improve input data quality from both perspectives.  For topology, we observe that higher classification accuracy can be achieved when the ratio of inter-class edges (connecting nodes from different classes) is low and propose \textit{topology update} to remove inter-class edges and add intra-class edges. For node labels, we propose \textit{training node augmentation}, which enlarges the training set using the labels predicted by existing GNN models. SEG is a general framework that can be easily combined with existing GNN models. Experimental results validate that SEG consistently improves the performance of well-known GNN models such as GCN, GAT and SGC across different datasets.
\end{abstract}


% !TEX root = ./main.tex
\section{Introduction}

%Graph data are ubiquitous today, e.g., friendship graphs in social networks, phone call or message graphs in tele-communication, user-item interaction graphs in recommender systems, and protein-protein interaction graphs in biology. For graph-based tasks such as node classification, link prediction and graph classification, \textit{graph neural networks} (\textit{GNNs}) achieve excellent performance thanks to its ability to utilize both graph structure and feature information (on nodes or edges). Most GNN models can be formulated under the message passing framework, in which each node passes messages to its neighbors in the graph and aggregates messages from the neighbors to update its own embedding.

Graph data are ubiquitous, e.g., friendship graphs in social networks, user-item graphs in recommender systems, and protein-protein interaction graphs in biology. For graph-based tasks such as node classification, link prediction and graph classification, \textit{graph neural networks} (\textit{GNNs}) achieve excellent performance thanks to its ability to utilize both graph structure and feature information. Motivated by graph spectral theory, GCN~\citep{kipf2017semi} conducts graph convolution to avoid the high complexity of spectral decomposition. Instead of using the adjacency matrix to derive the weights for neighborhood aggregation, GAT~\citep{velickovic2018graph} uses an attention module to learn the weights from data. SGC~\citep{felix2019simplifying} removes the non-linearity in GCN as it observes that GCN performs well mainly because of neighborhood aggregation rather than non-linearity. There are many other GNN models such as GraphSAGE~\citep{hamilton2017inductive}, JK-Net~\citep{xu2018representation}, Geom-GCN~\citep{pei2020geomgcn},  GGNN~\citep{li2015gated}
% , GIN~\cite{xu2018how}
% , GMNN~\cite{qu2019gmnn} 
and ClusterGCN~\citep{chiang2019cluster}, and we refer readers to a comprehensive survey in~\citep{zhou2018graph}.

%Different attempts have been made to design algorithms and models for graph analytics. Random walk based methods, e.g., DeepWalk~\citep{perozzi2014deepwalk} uses the random walk paths as the input to a skip-gram model to learn node embeddings , while node2vec~\citep{grover2016node2vec} learns node embeddings by combining  breadth-first random walk and depth-first random walk. Motivated by graph spectral theory, graph convolutional network (GCN)~\citep{kipf2017semi} conducts graph convolution using the adjacency matrix to avoid the high complexity spectral decomposition. Instead of using the adjacency matrix to derive the weights for message aggregation, graph attention network (GAT)~\citep{velickovic2018graph} uses an attention module to learn the weights from data. Simplifying graph convolution network (SGC)~\citep{felix2019simplifying} proposes to remove the non-linearity in GCN as it observes that the good performance of GCN mainly comes from local averaging rather than non-linearity. There are also many other GNN models such as GraphSAGE~\citep{hamilton2017inductive}, jumping knowledge network (JK-Net)~\citep{xu2018representation}, geometric graph convolutional network (Geom-GCN)~\citep{pei2020geomgcn}, and gated graph neural network (GGNN)~\citep{li2015gated}, and we refer readers to a comprehensive survey in~\citep{zhou2018graph}.

In this paper, we focus on \textit{semi-supervised node classification}, which is the task that most GNN models are designed for. Most existing works propose more effective GNN models, but the quality of the input data has not received much attention. However, \textit{data quality}\footnote{We adopt a task-specific definition of data quality. Given a GNN model and a specific problem, high data quality means that the GNN model achieves good performance for the problem on the input data. In this paper, we discuss data quality w.r.t. the node classification problem.} and model quality can be equally important for good performance. For example, if the input graph contains only \textit{intra-class edges} (i.e., edges connecting nodes from the same class) and no \textit{inter-class edges} (i.e., edges connecting nodes from different classes), node classification can achieve perfect accuracy with only one training sample from each connected component. Moreover, classification are usually easier with more training samples.

%Here, data quality is problem-specific. Given a GNN model and a specific problem, high data quality means that the GNN model produces good output for the problem on the input data. In this paper, we discuss data quality with respect to the node classification problem.

%At first glance, data quality (i.e., the quality of the input graph structure and training nodes) is the fixed problem input and cannot be improved. However, we observed that existing GNN models already achieve good classification accuracy, and thus their outputs can actually be used to update the input data to improve its quality. Then, the GNN models can be trained on the improved data to achieve better performance. We call this idea \textbf{self-enhanced GNN} and propose two algorithms under this framework, namely \textbf{topology update (TU)} and \textbf{training node augmentation (TNA)}.  

At first glance, data quality is fixed with the input data to a problem and cannot be improved. However, we observed that existing GNN models already achieve good classification accuracy, and thus their outputs can be used to update the input data to improve its quality. Then, the GNN models can be trained on the improved data to achieve better performance. We call this idea \textbf{self-enhanced GNN} (SEG) and propose two algorithms under this framework, namely \textbf{topology update (TU)} and \textbf{training node augmentation (TNA)}.

%in a co-training or self-training fashion

%As GNN models essentially smooth the embeddings of neighboring nodes, inter-class edges harm the performance as they make it difficult to distinguish nodes from different classes. To this end, TU removes inter-class edges and adds intra-class edges according to node labels predicted by a GNN model. Our analysis shows that TU reduces the percentage of inter-class edges in the input graph as long as the performance of the GNN model is good enough. Since the number of labeled nodes are usually small for semi-supervised node classification, TNA enlarges the training set by treating the predicted labels of multiple GNN models as the ground truth. We show by analysis that jointly considering the predicted labels of multiple diverse GNN models reduces errors in the enlarged training set. We also develop a method to create diversity among multiple GNN models. In addition, we propose techniques such as \textit{threshold-based selection}, \textit{validation-based tuning} and \textit{class balance} to stabilize the performance of TU and TNA. Both TU and TNA are general techniques that can be easily combined with existing GNN models. 

As GNN models essentially smooth the embeddings of neighboring nodes~\cite{nt2019revisiting}, inter-class edges can be harmful to the model performance as they make it difficult to distinguish nodes from different classes. To this end, TU removes inter-class edges and adds intra-class edges according to node labels predicted by a GNN model. Our analysis shows that TU reduces the percentage of inter-class edges in the graph as long as the performance of the GNN model is good enough. Since the number of labeled nodes are usually small for semi-supervised node classification, TNA enlarges the training set by treating the predicted labels of multiple GNN models as the ground truth. We show by analysis that using multiple diverse GNN models reduces errors in the enlarged training set. We also develop an effective method to create diversity among multiple GNN models. Both TU and TNA are general techniques that can be easily combined with existing GNN models. 

%jointly considering the predicted labels of
%We propose techniques such as \textit{threshold-based selection}, \textit{validation-based tuning} and \textit{class balance} to stabilize the performance of TU and TNA.

We conducted extensive experiments on three well-known GNN models, i.e., GCN, GAT and SGC, and seven widely used benchmark datasets. The results show that SEG consistently improves the performance of these GNN models. The reduction in the classification error is 16.2\% on average and can be up to 35.1\%. Detailed profiling finds that TU and TNA indeed improve the input data quality for node classification. Specifically, TU effectively deletes inter-class edges and adds intra-class edges, while most of the nodes added by TNA are assigned a right label. Based on these results, one interesting future direction is to extend the idea of SEG to other problems such as link prediction and graph classification where GNNs are also used. 

% Original:
% \noindent \textbf{Relation with existing work.} SEG is a form of co-training~\citep{blum1998combining} or self-training~\citep{yarowsky1995self-training} for GNNs. AdaEdge~\citep{chen2019measuring} also removes/adds edges for GNN but the goal is to mitigate over-smoothing and support more graph convolution layers. SEG tries to improve data quality and observes that lower noise ratio leads to higher classification accuracy. In addition, our analysis shows that adding/removing edges can reduce noise ratio if the performance of a model is good enough. \cite{li2018deeper} enlarge the training set for GNN using label propagation and trained GNN model. As they use a single model for TNA, their performance can even be worse than the baseline model. In SEG, we show by analysis that the diversity among different models is crucial for TNA and propose an effective method to generate multiple diverse models. By combining TU and TNA, SEG provides consistent performance improvement for all the models in our experiments.  

% !TEX root = ./main.tex

\begin{figure}[t]	
	\centering
	\begin{subfigure}{0.49\linewidth}
		\centering 
		\includegraphics[width=\linewidth]{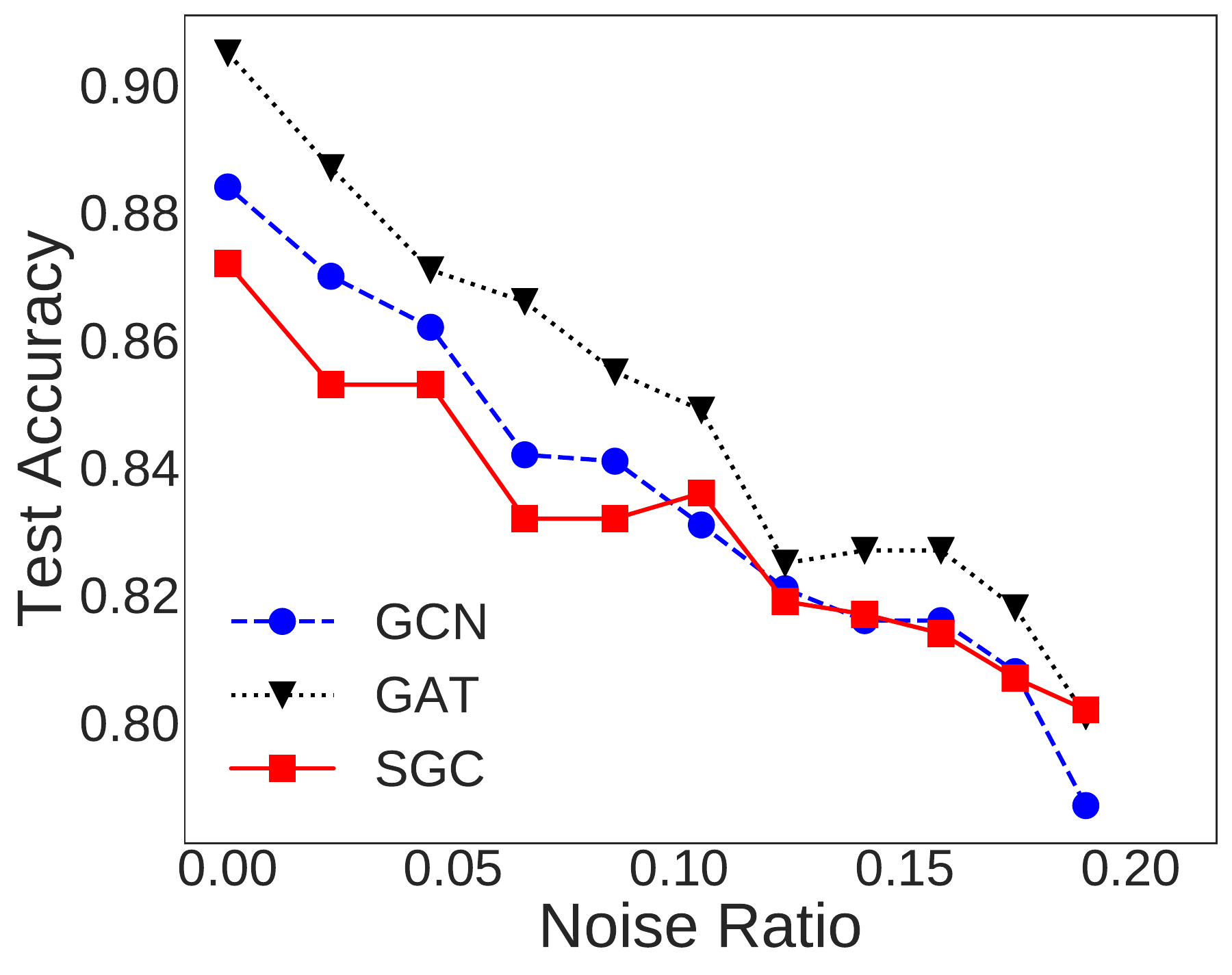} 
		\caption{Edge deletion}
		\label{fig:delete edge}
	\end{subfigure}
	\begin{subfigure}{0.49\linewidth}
		\centering 
		\includegraphics[width=\linewidth]{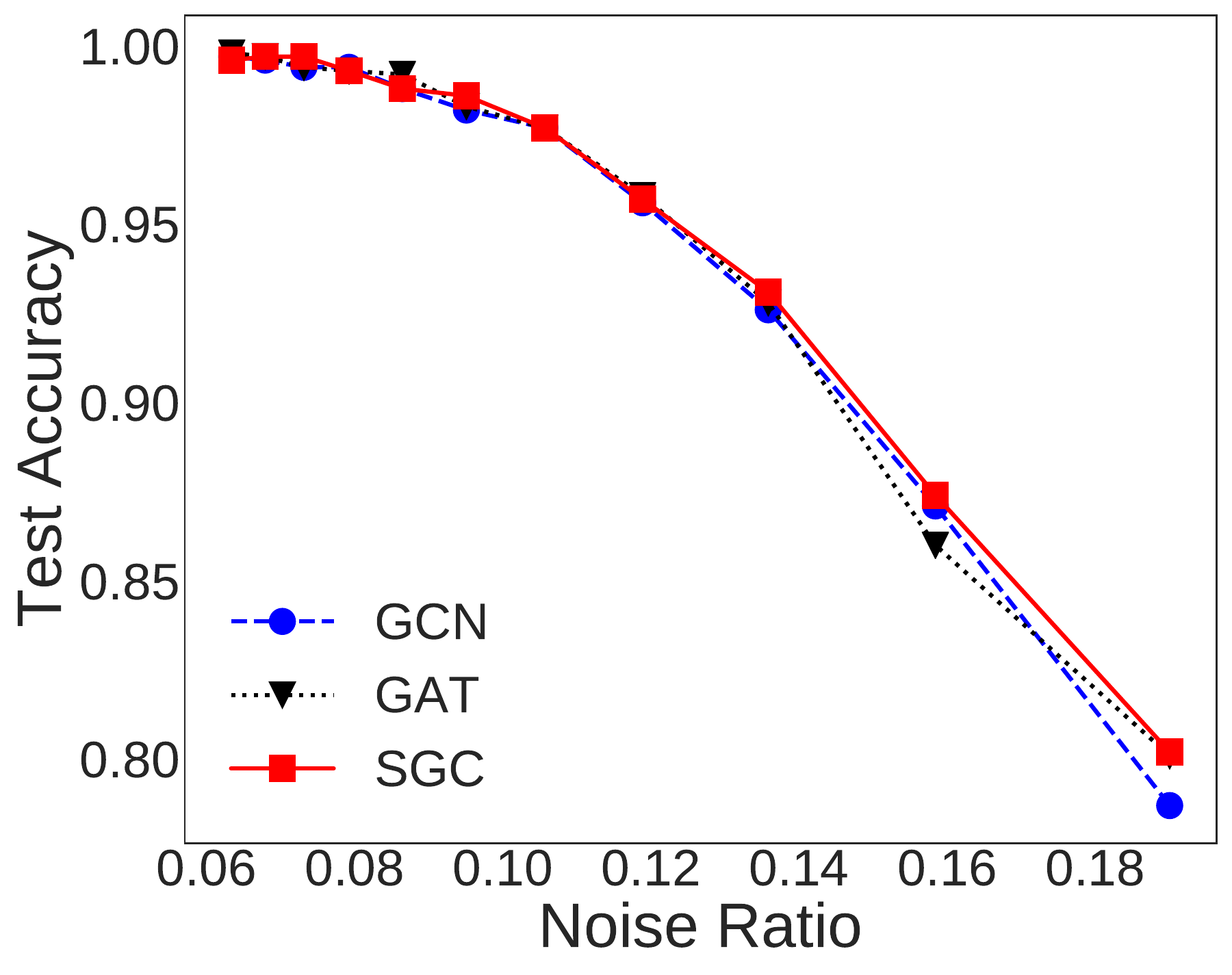} 
		\caption{Edge addition}
		\label{fig:add edge}
	\end{subfigure}
	\caption{The relation between \textit{noise ratio} and \textit{test accuracy} on the \cora dataset, note that both edge deletion and edge addition reduce noise ratio (from right to left on  the x-axis). Similar trend is also observed on other datasets} 
	\label{fig:noise ratio}
	% \vspace{-3mm}
\end{figure}

\noindent \textbf{Relation with existing work.} SEG can be seen as a general framework for  co-training~\cite{blum1998combining} or self-training~\cite{yarowsky1995self-training} of GNNs, including both graph structures and node labels. Related work such as \cite{chen2019measuring, li2018deeper, sun2020multi} can be seen as specific algorithms for self-training or co-training of GNNs, which can also be incorporated into the SEG framework. The designs in SEG are  different from the above-mentioned works in the following aspects. \citep{chen2019measuring} removes/adds edges for GNN to mitigate over-smoothing and to support more graph convolution layers; in contrast, we attempt to improve data quality and observe that lower noise ratio leads to higher classification accuracy. \cite{li2018deeper} enlarges the training set for GNN using label propagation and a single trained GNN model, \cite{sun2020multi} adopts unsupervised learning techniques like DeepCluster~\cite{caron2018deep} to help the self-training of GNNs, and \cite{yang2020rethinking} uses a regularization term to make the model predictions of each node's neighbors to supervise itself. In contrast, in SEG we show by analysis that the diversity among multiple different models is crucial for the \textit{training node augmentation}  algorithm and propose an effective method to generate multiple diverse models. In addition, the theoretical foundations about noise in edges and model diversities developed in the SEG framework are general and  helpful to the above works as they can be considered as variants of \textit{topology update} or \textit{training node augmentation} algorithms. In graph adversarial attack literature, there are works that add/remove edges to deteriorate the accuracy of the GNN models~\cite{vassilis2019edge,wu2019adversarial,zugner2018adversarial}. Our \textit{topology update} algorithm is for an opposite purpose, which is to improve data quality and model accuracy.

\section{Topology Update}\label{sec:topology update}

%\subsection{Motivation}

%\textit{lower noise ratio leads to higher classification accuracy }

%There are $n$ nodes and $m$ edges in the graph.

Denote a graph as $\mathcal{G}=(\mathcal{V}, \mathcal{E})$, where $\mathcal{V}$ is the set of $n$ nodes and $\mathcal{E}$ is the set of $m$ edges. The ground-truth label of a node $v$ is $l(v)$. We define the \textit{noise ratio} of the graph $\mathcal{G}$ as
\begin{equation}
\alpha=\frac{|\left\{l(u)\neq l(v) | e_{uv}\in\mathcal{E} \right\}|}{|\mathcal{E}|}.
\end{equation}
%where $|.|$ is the cardinality of a set. 
Noise ratio measures the percentage of \textit{inter-class edges} (i.e., $e_{uv}$ with $l(u)\neq l(v)$) in the graph. 

\noindent \textbf{Motivation.} In Figure~\ref{fig:noise ratio}, we show the relation between \textit{classification accuracy} and \textit{noise ratio} for the \cora dataset, where \textit{edge deletion} randomly removes inter-class edges in the graph and \textit{edge addition} randomly adds \textit{intra-class edges} (i.e., $e_{uv}$ with $l(u)=l(v)$) based on the ground-truth labels. The results show that the classification accuracy is higher for all the three models under lower noise ratio. This is understandable since GNN models are generally low-pass filters that smooth the embeddings of neighboring nodes~\citep{nt2019revisiting}. As inter-class edges encourage nodes from different classes to have similar embeddings, they make the classification task difficult. Therefore, we make the following assumption.

\begin{assumption}
	Lower noise ratio leads to better classification performance for GNN models.
\end{assumption}

%	Lower noise ratio leads to better classification performance for popular GNNs such as GCN, GAT and SGC.

%\subsection{Topology Update Algorithms}

%For Figure~\ref{fig:noise ratio}, we delete/add edges using the ground-truth labels. However, we may not have access to the ground-truth labels in a practical node classification problem. As popular GNN models already provide quite accurate predictions of the true labels, we can use their output for edge edition. Denote a GNN model trained for a node classification problem with $c$ classes as a mapping function $f:\mathbb{V} \rightarrow [c] $, where $[c]$ is the integer set $\left\{1,\ldots, c\right\}$. Edge deletion and edge addition can be conducted using Algorithm~\ref{alg:edge deletion} and Algorithm~\ref{alg:edge addition}, respectively. 

\noindent\textbf{Topology Update algorithms.} For Figure~\ref{fig:noise ratio}, we delete/add edges using the ground-truth labels. However, we do not have access to the ground-truth labels in a practical node classification problem. As popular GNN models already provide quite accurate predictions of the true labels, we can use their output for edge edition. Denote a GNN model trained for a node classification problem with $c$ classes as a mapping function $f:\mathcal{V} \rightarrow [c] $, where $[c]$ is the integer set $\left\{1,\ldots, c\right\}$. Edge deletion and edge addition can be conducted using Algorithm~\ref{alg:edge deletion} and Algorithm~\ref{alg:edge addition}, respectively.

\begin{algorithm}[t]
	\caption{Edge Deletion}
	\label{alg:edge deletion}
	\begin{algorithmic}
	\STATE {\bfseries Input:} A graph $\mathcal{G}=(\mathcal{V}, \mathcal{E})$ and a trained GNN model $f(\cdot)$ 
	\STATE {\bfseries Output:} A new graph $\mathcal{G}'=(\mathcal{V}, \mathcal{E}')$ 
	\STATE Initialize $\mathcal{E}'=\mathcal{E}$; 
	\FOR{each edge $e_{uv}\in\mathcal{E}'$}
	\IF{$f(u) \neq f(v)$ }
	\STATE Delete $e_{uv}$ from $\mathcal{E}'$; 
	\ENDIF  
	\ENDFOR	
	\end{algorithmic}
\end{algorithm}
\begin{algorithm}[t]
\caption{Edge Addition}
\label{alg:edge addition}
	\begin{algorithmic}
	\STATE {\bfseries Input:} A graph $\mathcal{G}=(\mathcal{V}, \mathcal{E})$ and a trained GNN model $f(\cdot)$ 
	\STATE {\bfseries Output:} A new graph $\mathcal{G}'=(\mathcal{V}, \mathcal{E}')$ 
	\STATE Initialize $\mathcal{E}'=\mathcal{E}$;
	\FOR{each node pair $(u, v)\in \mathcal{V} \times \mathcal{V}$}
	\IF{$e_{uv}\notin \mathcal{E}'$ and $f(u)=f(v)$ }
	\STATE Add $e_{uv}$ to $\mathcal{E}'$; 
	\ENDIF  
	\ENDFOR	
	\end{algorithmic}
\end{algorithm}

\begin{figure}[t]
	\centering
	\begin{subfigure}{0.49\linewidth}
		\centering 
		\includegraphics[width=\textwidth]{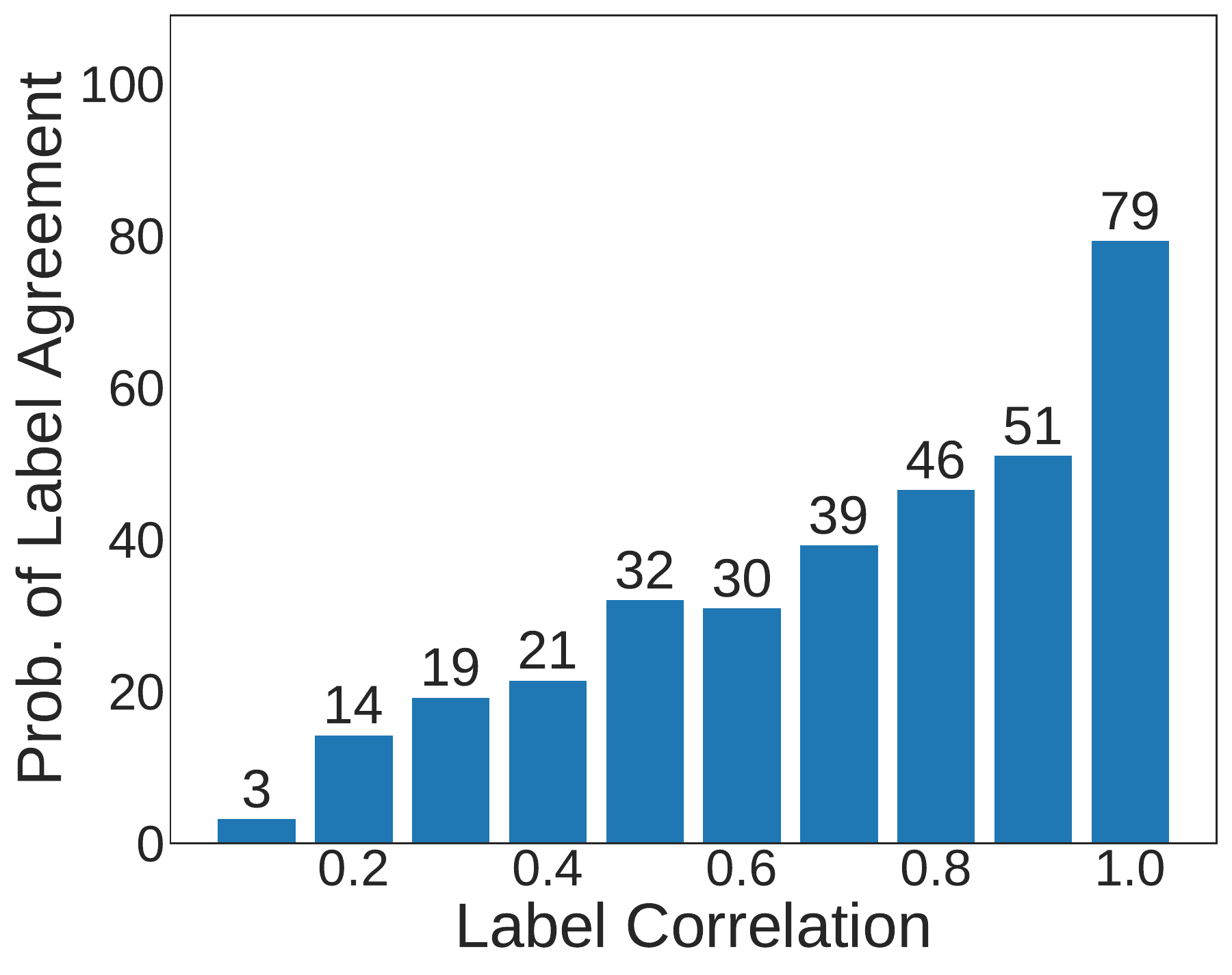} 
		\caption{\cora}
	\end{subfigure}
	\begin{subfigure}{0.49\linewidth}
		\centering 
		\includegraphics[width=\textwidth]{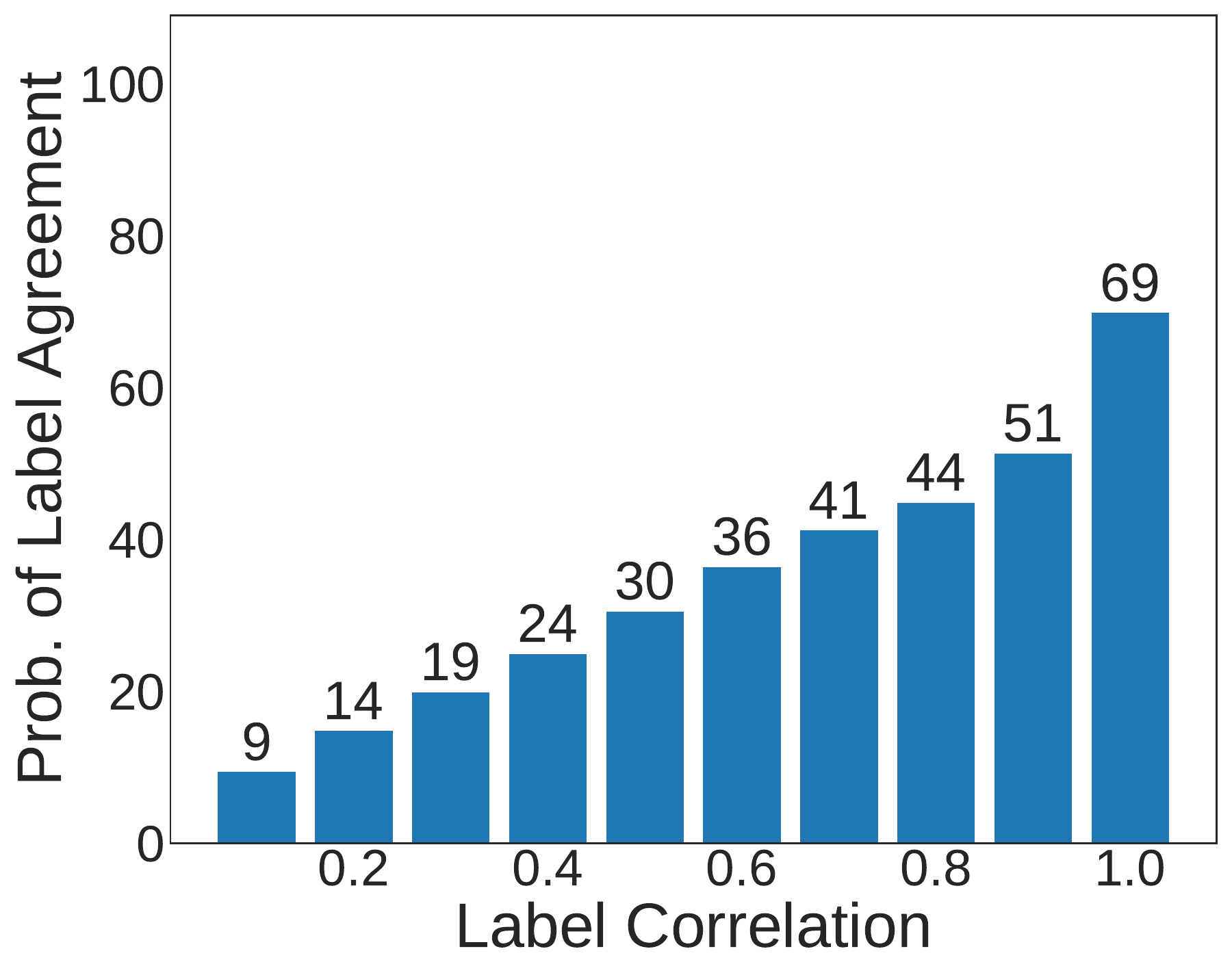} 
		\caption{\pubmed}
	\end{subfigure}
	\caption{The relation between label correlation (i.e., $g_v^{\top}g_u$) and the probability of having the same label for GCN}
	\label{fig:inter_prod_same_label}
	% \vspace{-3mm}
\end{figure}

\noindent\textbf{Analysis.} In the following analysis, we show that Algorithm~\ref{alg:edge deletion} and Algorithm~\ref{alg:edge addition} reduce the noise ratio of the input graph if the classification accuracy of the GNN model $f(\cdot)$ is high enough. We first present some assumptions and definitions that will be used in the analysis.

\begin{assumption}\label{assup:symetric}
	(Symmetric Error) The GNN model $f(\cdot)$ has a classification accuracy of $p$ and makes symmetric errors, i.e., for every node $v \in \mathcal{V}$, we have $\mathbb{P}[f(v)=l(v)]=p$ and $\mathbb{P}[f(v)=k]=\frac{1-p}{c-1}$ for $k\in [c]$ and $ k\neq l(v)$, where $l(v)$ is the ground-truth label of node $v$.
\end{assumption} 

Note that symmetric error is a common assumption in the literature~\citep{chen2019ncv} and our analysis methodology is not limited to symmetric error. As the GNN model $f(\cdot)$ makes random errors (and hence the topology update algorithms also make random errors), we use the expected noise ratio $\alpha_{E}$ for the updated graph $\mathcal{G}'$ as a replacement for the noise ratio $\alpha$. For $\mathcal{G}'=(\mathcal{V}, \mathcal{E}')$, we define the expected noise ratio as $\alpha_{E}=\frac{m_{r}}{m_{r}+m_{a}}$, in which $m_{r}$ is the expected number of inter-class edges in $\mathcal{G}'$ and $m_{a}$ is the expected number of intra-class edges in $\mathcal{G}'$. We can compare the expected noise ratio of $\mathcal{G}'$ with the noise ratio of the original graph $\mathcal{G}$.

\begin{theorem}\label{th:deletion}  
	(Edge Deletion) If Assumption~\ref{assup:symetric} holds and Algorithm~\ref{alg:edge deletion} is used for edge deletion, denote the expected noise ratio of the output graph $\mathcal{G}'=(\mathcal{V}, \mathcal{E}')$ as $\alpha_{E}$, we have $\alpha_{E}< \alpha$ if  $p>\frac{2}{c+1}$.    
\end{theorem}

%\begin{proof}
%	The probability that an intra-class edge in $\mathcal{G}$ is kept in $\mathcal{G}'$ by Algorithm~\ref{alg:edge deletion} is $p_a=\mathbb{P}\left[f(v)=f(u)|l(u)=l(v)\right]=p^2+\frac{(1-p)^2}{c-1}$. Therefore, $m_{a}=(1-\alpha) m \left(p^2+\frac{(1-p)^2}{c-1}\right)$, where $m$ is the number of edges in $\mathcal{G}$. The probability that an inter-class edge is kept is
%	$p_r=\mathbb{P}\left[f(v)=f(u)|l(u)\neq l(v)\right]=\frac{2p(1-p)}{c-1}+\frac{(c-2)(1-p)^2}{(c-1)^2}$, and thus $m_{r}=\alpha m\left(\frac{2p(1-p)}{c-1}+\frac{(c-2)(1-p)^2}{(c-1)^2}\right)$. We have
%	\begin{equation*}
%	\begin{aligned}
%	\alpha_E&=\frac{\alpha m\left(\frac{2p(1-p)}{c-1}+\frac{(c-2)(1-p)^2}{(c-1)^2}\right)}{\alpha m\left(\frac{2p(1-p)}{c-1}+\frac{(c-2)(1-p)^2}{(c-1)^2}\right)+(1-\alpha) m \left(p^2+\frac{(1-p)^2}{c-1}\right)}\\
%	&<\frac{\alpha(1-p^2)}{\alpha(1-p^2)+(1-\alpha)[(c-1)p^2+(1-p)^2]}.
%	\end{aligned}
%	\end{equation*}  
%	Solving $\frac{\alpha(1-p^2)}{\alpha(1-p^2)+(1-\alpha)[(c-1)p^2+(1-p)^2]}<\alpha$ gives $(1-\alpha)p[(c+1)p-2]\ge 0$, which is satisfied when $p>\frac{2}{c+1}$.    
%\end{proof} 

All proofs can be found in Section A of the supplementary material\footnote{\url{https://arxiv.org/pdf/2002.07518.pdf}}. Theorem~\ref{th:deletion} shows that edge deletion reduces noise ratio under a mild condition on the classification accuracy of the GNN model, i.e., $p\!>\!\frac{2}{c+1}$. For example, for a node classification problem with 5 classes, it only requires that $p\!>\!1/3$. To analyze the expected noise ratio of the graph after edge addition, we further assume that \textit{the classes are balanced}, i.e., each class has $n/c$ nodes.

%\begin{lemma}\label{lemma:xx}
%	If Assumption~\ref{assup:symetric} holds and the classes are balanced, for a node $v\in\mathcal{G}$, we have $\mathbb{P}[l(v)=f(v)|f(v)]=p$ and $\mathbb{P}[l(v)=j|f(v)]=\frac{1-p}{c}$ for $ j\neq f(v)$ and $1 \le j \le c$.
%\end{lemma}         

\begin{theorem}\label{th:addition}
	(Edge Addition) If Assumption~\ref{assup:symetric} holds, the classes are balanced in $\mathcal{G}$, and Algorithm~\ref{alg:edge addition} is used for edge addition, denote the expected noise ratio of the output graph $\mathcal{G}'=(\mathcal{V}, \mathcal{E}')$ as $\alpha_{E}$, we have $\alpha_{E}<\alpha$ if $p>\frac{\alpha+\sqrt{\alpha^2+[(c-1)(1+c\alpha \lambda)-c\alpha](c+\alpha-1)}}{c+\alpha-1}$, in which $\lambda=\frac{m}{n^2}$ is the edge density of $\mathcal{G}$.     
\end{theorem}

%
%\begin{proof}	 
%	Denote the expected number of added intra-class edges as $m'_{a}$ and the expected number of added inter-class edges as $m'_{r}$. To ensure $\alpha_{E}<\alpha$, it suffices to show that $\frac{m'_{r}}{m'_{a}+m'_{r}} < \alpha$. As there are $\frac{c-1}{c} n^2$ possible inter-class edges and $\frac{1}{c} n^2$ intra-class edges in $\mathcal{V} \times \mathcal{V}$, we have
%	\begin{equation*}
%	\begin{aligned}
%	m'_{r}&=(\frac{c-1}{c} n^2-m \alpha)p_{r}<\frac{c-1}{c} n^2 p_{r}\\
%	m'_{a}&=\left[\frac{1}{c} n^2-m (1-\alpha)\right]p_{a}>\frac{1}{c} n^2 p_{a}-m,
%	\end{aligned} 
%	\end{equation*} 
%	where $p_{r}$ and $p_{a}$ are the probability of keeping an inter-class edge and an intra-class edge in $\mathcal{G}'$, respectively. Their expressions are given in the proof of Theorem~\ref{th:addtion}. The $m \alpha$ and $m (1-\alpha)$ terms are included to exclude the overlaps between the edges in the original graph and the edges that may be added by Algorithm~\ref{alg:edge addition}. With $m=n^2 \lambda$, we have  
%	\begin{equation*}
%	\frac{m'_{r}}{m'_{a}+m'_{r}}<\frac{\frac{c-1}{c} n^2 p_{r}}{\frac{c-1}{c} n^2 p_{r}+\frac{1}{c} n^2 p_{a}-m}<\frac{1-p^2}{1+\frac{(1-p)^2}{c-1}-c\lambda}.
%	\end{equation*} 
%	Solving $\frac{1-p^2}{1+\frac{(1-p)^2}{c-1}-c\lambda}<\alpha$ gives the result. 
%\end{proof}

The bound on $p$ in Theorem~\ref{th:addition} is complex for interpretation but we can approximate it as $p>\frac{\alpha+\sqrt{\alpha^2+(c+\alpha-1)(c-1-c\alpha)}}{c-1}$ if we assume that the $\lambda$ term is small enough to be ignored. The bound can be further simplified as $p>\sqrt{1-\alpha}$ if we assume that $\alpha$ is small compared to $c$ and approximate $c-1-c\alpha$ with $(c-1)(1-\alpha)$. 

%Note that $p >\sqrt{1-\alpha}$ is a higher requirement on the classification accuracy of $f(\cdot)$ than  $p >\frac{2}{c+1} $ for edge deletion. Thus, as we will show in the experiments, the performance improvement of edge addition is usually smaller than edge deletion. 

Theorem~\ref{th:deletion} and Theorem~\ref{th:addition} can be extended to more general assumptions. The symmetric error assumption can be replaced with an error matrix $\bm{E}\in\mathbb{R}^{c\times c}$, where $\bm{E}(i,j)$ is the probability of classifying class $i$ as class $j$. The number of nodes in each class can also be different. The analysis methodology in our proofs can still be applied but the bounds on $p$ will be in more complex forms. In addition, we show in the experiments that edge deletion and addition can be conducted simultaneously.

\noindent \textbf{Threshold-based update selection.} The GNN model $f(\cdot)$ usually outputs a distribution over the classes (e.g., using softmax) rather than a single decision. For a node $v$, we denote its class distribution provided by the model as $g_v\in\mathbb{R}^{c}$ with $g_v[k]\ge 0$ for $k \in [c]$ and $\sum_{k=1}^{c} g_v[k] =1$. In Figure~\ref{fig:inter_prod_same_label}, we plot the relation between label correlation and the probability that a pair of nodes have the same label (called node alignment). The results show that a  pair of nodes is more likely to be in the same class under higher label correlation. Thus in practice, for edge deletion, we first generate a candidate edge set $\mathcal{C}$ based on the classification labels using Algorithm~\ref{alg:edge deletion}. For each candidate edge $e_{uv}$ in $\mathcal{C}$, we calculate the correlation between their class distributions (i.e., $g_u^{\top}g_v$) and select the edges with $g_u^{\top}g_v\le \tau_d$ for actual deletion, where $\tau_d$ is a threshold. For edge addition, we also generate a candidate set using Algorithm~\ref{alg:edge addition} first and add only edges with $g_v^{\top}g_u\ge \tau_a$. Threshold-based update selection makes Algorithm~\ref{alg:edge deletion} and Algorithm~\ref{alg:edge addition} more conservative, which helps avoid deleting intra-class edges and adding inter-class edges.  We use the test accuracy on the validation set to tune the thresholds $\tau_d$ and $\tau_a$.

\begin{figure}[t]
	\centering
	\begin{subfigure}{0.49\linewidth}
		\centering 
		\includegraphics[width=\textwidth]{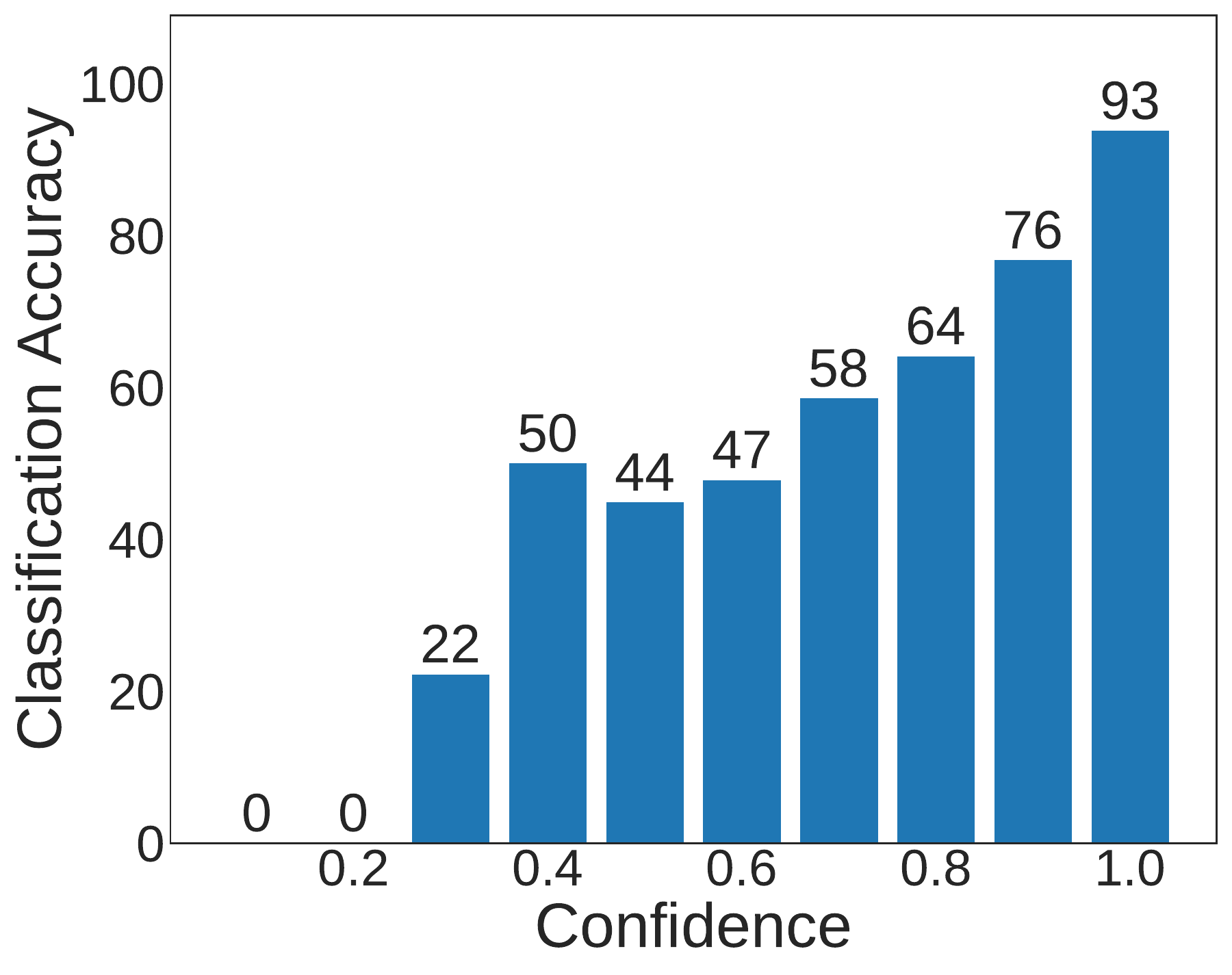}
		\caption{\cora}
	\end{subfigure}
	\begin{subfigure}{0.49\linewidth}
		\centering 
		\includegraphics[width=\textwidth]{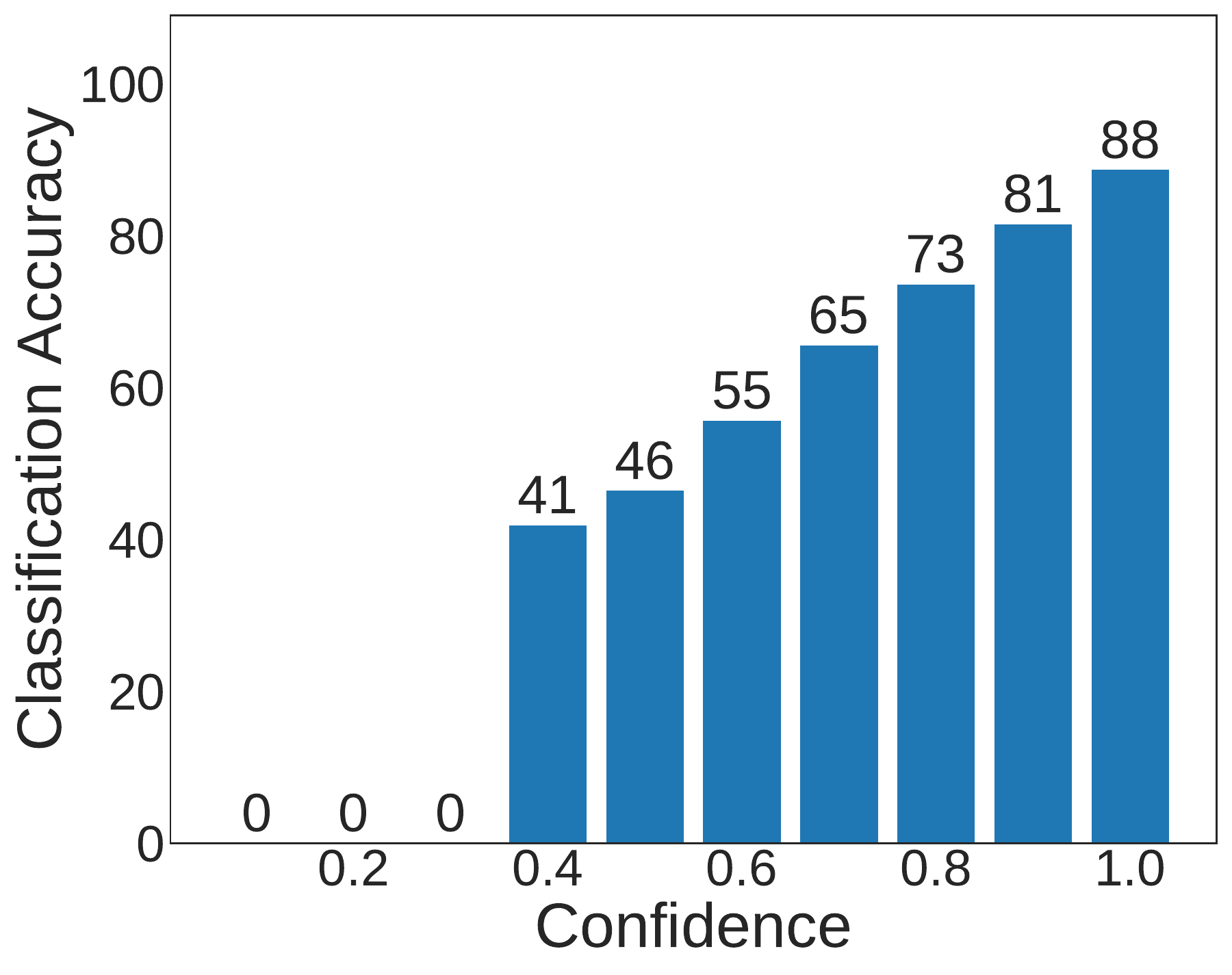} 
		\caption{\pubmed}
	\end{subfigure}
	\caption{The relation between the confidence score $c_v$ and the probability of giving the right label prediction for GCN}
	\label{fig:confidence_acc}
	% \vspace{-4mm}
\end{figure}
\section{Training Node Augmentation}\label{sec:training node augmentation}

\noindent \textbf{Motivation.} In Figure~\ref{fig:sample and accuracy}, we observe the influence of the number of training nodes on classification accuracy. The results show that using more training nodes consistently leads to higher classification accuracy for all three models. Unfortunately, for semi-supervised node classification, usually only a very small number of labeled nodes are available. To enlarge the training set, an intuitive idea is to train a GNN model to label some nodes and add those nodes to the training set. However, a GNN model usually makes a considerable amount of errors in its label prediction, and naively using the predicted labels as the ground-truth labels may lead to worse performance.   

\begin{figure}[t]
	\centering
	\includegraphics[width=0.49\linewidth]{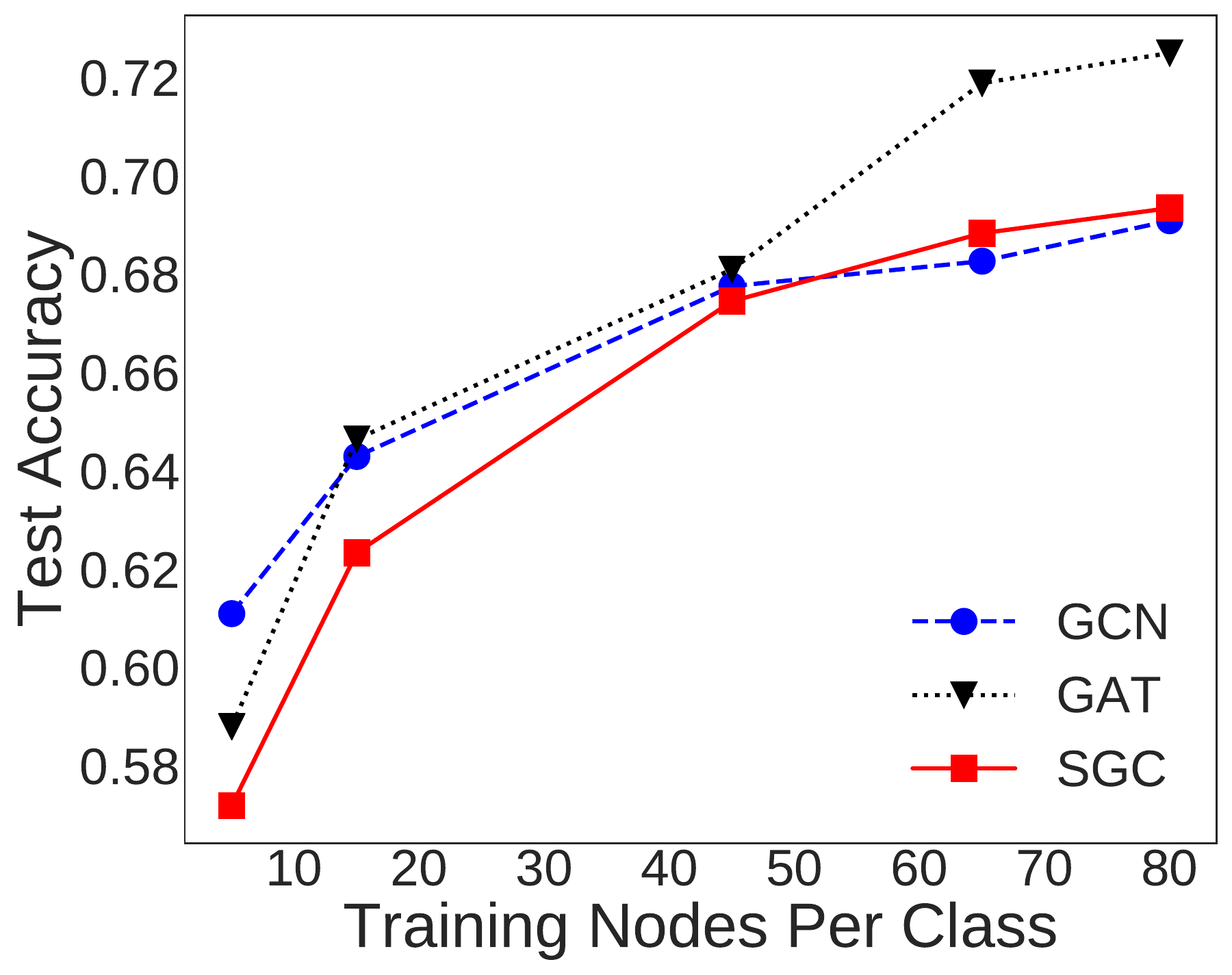} 
	\includegraphics[width=0.49\linewidth]{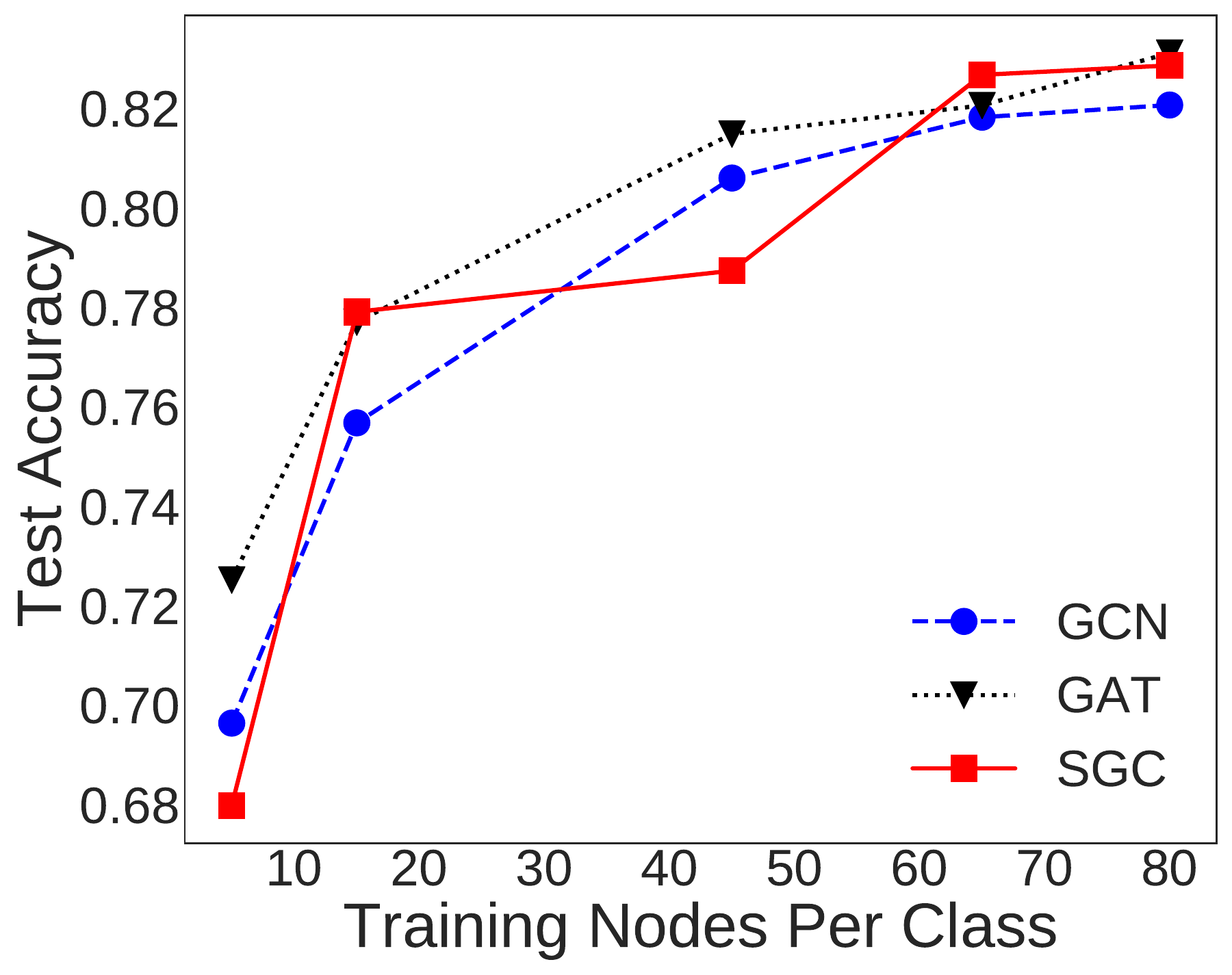}
	\caption{The relation between the number of training nodes and test accuracy for CiteSeer (left) and Cora (right)} 
	\label{fig:sample and accuracy}
	% \vspace{-4mm}
\end{figure}

\begin{figure}[t]
	\centering
	\includegraphics[width=0.49\linewidth]{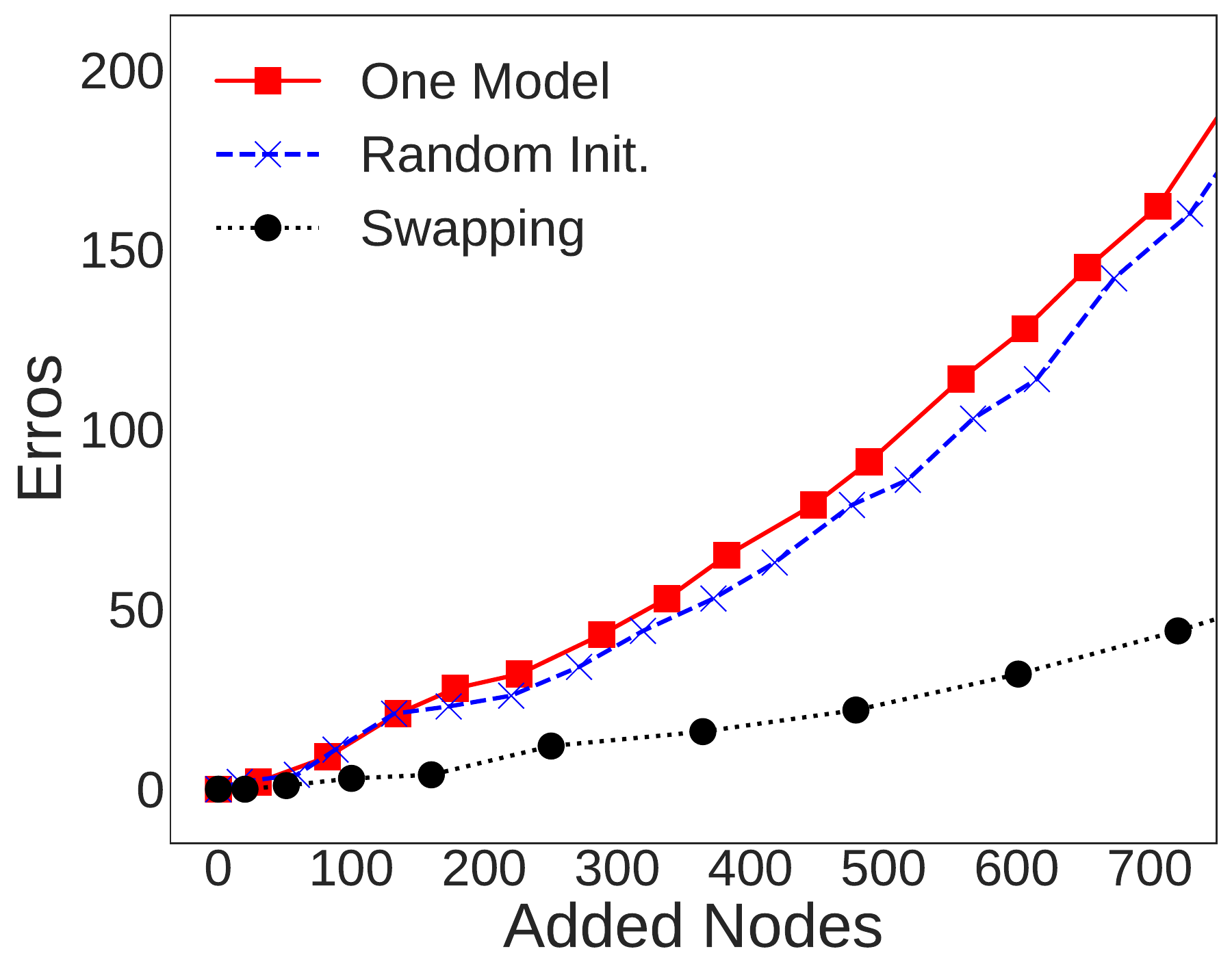} 
	\includegraphics[width=0.49\linewidth]{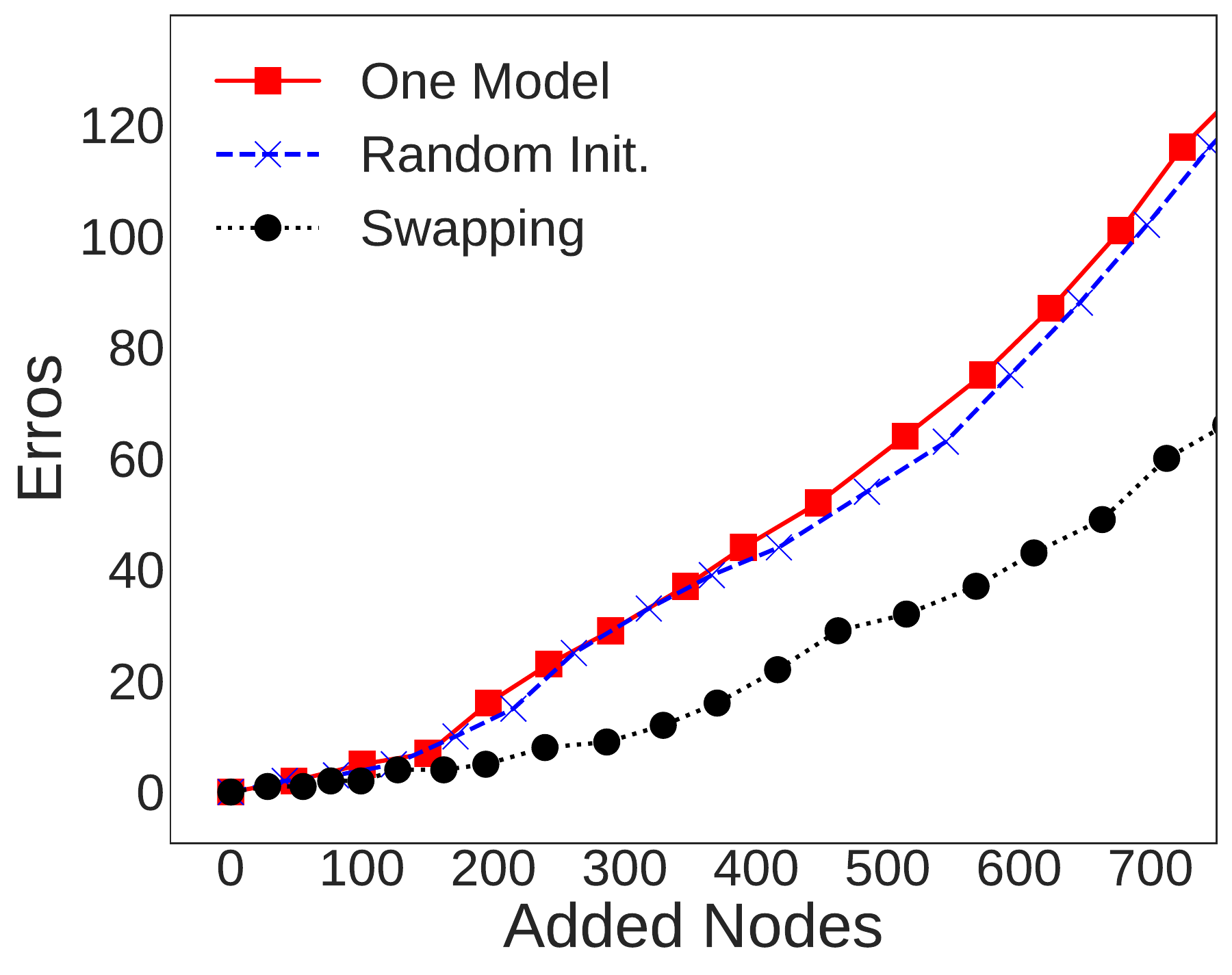} 
	\caption{Errors in $\mathcal{T}'$ for CiteSeer (left) and PubMed (right) using different methods (with the GCN model)} 
	\label{fig:diversity}
	% \vspace{-4mm}
\end{figure}
     
% using 1 model, 2 models with different random initializations, and 2 models with train-validation set swapping (the model is GCN)
%\begin{figure}
%	\centering
%	\begin{subfigure}{0.24\linewidth}
%		\centering 
%		\includegraphics[width=\textwidth]{./figures/nodes-citeseer.pdf} 
%		\caption{\citeseer}
%		\label{CiteSeer dataset}
%	\end{subfigure}
%	\begin{subfigure}{0.24\linewidth}
%		\centering 
%		\includegraphics[width=\textwidth]{./figures/nodes-cora.pdf} 
%		\caption{\cora}
%		\label{Cora dataset}
%	\end{subfigure}
%	\caption{The relation between the number of training samples and test accuracy for GCN, GAT and SGC} 
%	\label{fig:sample and accuracy}
%\end{figure}

%\begin{figure}
%	\centering
%	\begin{subfigure}{0.25\linewidth}
%		\centering 
%		\includegraphics[width=\textwidth]{./figures/error-citeseer.pdf} 
%		\caption{\citeseer}
%		\label{fig:CiteSeer}
%	\end{subfigure}
%	\begin{subfigure}{0.25\linewidth}
%		\centering 
%		\includegraphics[width=\textwidth]{./figures/error-pubmed.pdf} 
%		\caption{\pubmed}
%		\label{fig:Pubmed}
%	\end{subfigure}
%	\caption{Errors in $\mathcal{T}'$ using 1 model, 2 models with different random initializations, and 2 models with train-validation set swapping (the model is GCN)} 
%	\label{fig:diversity}
%\end{figure}

% for each node $v\in\mathcal{G}$

\noindent\textbf{Training Node Augmentation algorithm.} Considering a GNN model $g(\cdot)$ that outputs a probability distribution over $c$ classes, we define the \textit{confidence} ($c_v$) and \textit{prediction result} ($r_v$) of node $v$ as
\begin{equation*}
c_v=\max_{1\le k\le c}g_v[k] \quad \text{and} \quad r_v=\arg\max_{1\le k\le c}g_v[k],
\end{equation*}
where $r_v$ is the label of $v$ predicted by $g(\cdot)$ and $c_v$ is the likelihood of $r_v$. Usually $r_v$ is more likely to be correct (i.e., $r_v=l(v)$) when $c_v$ is large. This is supported by the Figure~\ref{fig:confidence_acc}, in which we plot the relation between confidence and classification accuracy. The results show that the model is more likely to give the right label prediction when the confidence is higher.  Utilizing $c_v$ and $r_v$, we present the training node augmentation (TNA) procedure in Algorithm~\ref{alg:augmentation}, which produces an enlarged training set $\mathcal{T}'$ using the outputs of multiple GNN models. In Algorithm~\ref{alg:augmentation}, $\mathcal{T}$ and $\mathcal{S}$ denote the original training set and validation set. Before adding a node to $\mathcal{T}'$, we check if it is already in $\mathcal{T}$ and $\mathcal{S}$ to avoid assigning a new label to nodes in the two sets. Note that $\mathcal{T} \cup \mathcal{T}'$ is utilized to train a new model.

% \paragraph{Relation between confidence and classification accuracy.} In Algorithm~\ref{alg:augmentation}, we only add nodes with a high confidence $c_v$ into the enlarged training set $\mathcal{T}'$. In Figure~\ref{fig:confidence_acc}, we plot the relation between confidence and classification accuracy. The results show that the model is more likely to give the right label prediction with high confidence.     

%\begin{algorithm}[]
%	\caption{Training Node Augmentation}
%	\label{alg:augmentation}
%	\begin{algorithmic}
%		\STATE {\bfseries Input:} A graph $\mathcal{G}=(\mathcal{V}, \mathcal{E})$ and $L$ trained GNN models $g^1,g^2,\cdots,g^L$ 
%		\STATE {\bfseries Output:} An enlarged training set $\mathcal{T}'$ 
%		\STATE Initialize $\mathcal{T}'=\emptyset$;
%		\FOR{each model $g^l$}
%		\STATE Initialize candidate set $\mathcal{C}^l=\emptyset$;
%		\FOR{each node $v$ in $\mathcal{G}$}
%		\IF{$c^l_v\ge \tau_c $}
%		\STATE Add $v$ to $\mathcal{C}^l$;
%		\ENDIF 
%		\ENDFOR 
%		\ENDFOR
%		\STATE Candidate set $\mathcal{C}=\cap_{l=1}^{L} \mathcal{C}^l$; 
%		\FOR{each node $v$ in $\mathcal{C}$}
%		\IF{$v\notin \mathcal{T}$, $v\notin \mathcal{S}$ and $r^1_v=r^2_v=\cdots=r^L_v$ }
%		\STATE Add $v$ to $\mathcal{T}'$ with label $r^1_v$;
%		\ENDIF
%		\ENDFOR	
%	\end{algorithmic}
%\end{algorithm} 

\begin{algorithm}[t]
	\caption{Training Node Augmentation}
	\label{alg:augmentation}
	\begin{algorithmic}
		\STATE {\bfseries Input:} A graph $\mathcal{G}=(\mathcal{V}, \mathcal{E})$ and $L$ trained GNN models $g^1,g^2,\cdots,g^L$ 
		\STATE {\bfseries Output:} An enlarged training set $\mathcal{T}'$ 
		\STATE Initialize $\mathcal{T}'=\emptyset$;
		\FOR{each model $g^l$ in [$g^1,g^2,\cdots,g^L$]}
		\STATE Calculate confidence $c^l$ and prediction result $r^l$;
		\STATE Local candidate set $\mathcal{C}^l := \{v | c^l_v\ge \tau_c, v \in \mathcal{V} \}$;
		\ENDFOR
		\STATE Candidate set $\mathcal{C} := \cap_{l=1}^{L} \mathcal{C}^l$; 
		\FOR{each node $v$ in $\mathcal{C}$}
		\IF{$v\notin \mathcal{T}$, $v\notin \mathcal{S}$ and $r^1_v=r^2_v=\cdots=r^L_v$ }
		\STATE Add $v$ to $\mathcal{T}'$ with label $r^1_v$;
		\ENDIF
		\ENDFOR	
	\end{algorithmic}
\end{algorithm} 

%Algorithm~\ref{alg:augmentation} is based on two key ideas. The first one is only considering nodes with a high confidence (i.e., $c^l_v\ge \tau_c$) as the candidates to be added to $\mathcal{T}'$ since GNN models tend to produce more accurate label predictions at higher confidence. Similar to the case of topology update, we tune the value of $\tau_c$ based on the accuracy (of the model trained using $\mathcal{T}\cup\mathcal{T}'$) on the validation set. The second and most important idea is to utilize the \textit{diversity} of multiple GNN models to reduce the number of errors in $\mathcal{T}'$. With multiple diverse models, even if some classifiers assign a wrong label to node $v$, it will not be added to $\mathcal{T}'$ as long as one classifier gives the right label. In the following, we formalize this intuition with an analysis under the case of using two GNN models $g^1$ and $g^2$, i.e., $L=2$.

Algorithm~\ref{alg:augmentation} is based on two key ideas. The first one is only considering nodes with a high confidence (i.e., $c^l_v\ge \tau_c$) as the candidates to be added to $\mathcal{T}'$ since GNN models tend to produce more accurate label predictions at higher confidence. We use a threshold $\tau_c$ to select the high confidence nodes. The second and most important idea is to utilize the \textit{diversity} of multiple GNN models to reduce the number of errors in $\mathcal{T}'$. With multiple diverse models, considering when some classifiers assign a wrong label to node $v$, it will not be added to $\mathcal{T}'$ if at least one classifier gives the correct label. In the following, we formalize this intuition with an analysis under the case of using two GNN models $g^1$ and $g^2$.

\noindent\textbf{Analysis.} Following Assumption~\ref{assup:symetric}, we assume that both $g^1$ and $g^2$ have a classification accuracy of $p$ and make symmetric error. We also simplify Algorithm~\ref{alg:augmentation} and assume that a node is added to $\mathcal{T}'$ if the two models give the same label (i.e., $r^1_v=r^2_v$). Algorithm~\ref{alg:augmentation} can be viewed as a more conservative case of this simplified algorithm with $p'>p$ as it only adds high-confidence nodes. The accuracy of $\mathcal{T}'$ is defined as $q=\frac{|\left\{l(v)=r^1_v=r^2_v | v\in\mathcal{T}' \right\}|}{|\mathcal{T}'|}$. We are interested in the relation between $p$ and $q$, which are the accuracies of $\mathcal{T}'$ when using one model and two models for TNA, respectively. As the two models $g^1$ and $g^2$ are trained on the same graph structure, it is unrealistic to assume that they are independent. Therefore, we make the following assumption on how they correlate.

%$q=\frac{|\left\{v\in\mathcal{T}'|l(v)=r^1_v=r^2_v \right\}|}{\mathcal{T}'}$

\begin{assumption}\label{assumption:correlation}
(Model Correlation) The correlation between the two GNN models $g^1$ and $g^2$ can be formulated as
\[
	\begin{aligned}	
	&\begin{cases}
		\mathbb{P}[r^2_v=l(v)|r^1_v=l(v)]=\beta\\
		\mathbb{P}[r^2_v=k|r^1_v=l(v)]=\frac{1-\beta}{c-1}
		\end{cases}
		\text{and}\\
	&\begin{cases}
		\mathbb{P}[r^2_v=l(v)|r^1_v=k]=\gamma\\
		\mathbb{P}[r^2_v=j|r^1_v=k]=\frac{1-\gamma}{c-1}
		\end{cases},
	\end{aligned}
\]
where $k\in [c]$ and $ k\neq l(v)$, $j\in [c]$ and $ j\neq l(v)$. We also assume that $\beta\ge p$ as the two models should be positively correlated. 	
\end{assumption}

\begin{theorem}\label{th:label}
	(Accuracy)
	Under Assumption~\ref{assumption:correlation} and assume that $p>1/2$, we have the following results on the accuracy $q$ of  $\mathcal{T}'$
	
	(1) $q\ge p$; 
	
	(2) $q$ is maximized when $\beta=\gamma=p$, in which case the two models $g^1$ and $g^2$ are independent.     
\end{theorem}

Theorem~\ref{th:label} shows that using two models improves the accuracy of $\mathcal{T}'$ over using a single model. Moreover, we should make the GNN models independent to maximize the accuracy of $\mathcal{T}'$.

%\subsection{Optimizations for TNA}	\label{tna:opt}

%\noindent \textbf{Creating diversity in GNN models}. A straightforward method to generate multiple different GNN models is \textit{random initialization}, which trains the same model with different parameter initializations. We show the number of errors (i.e., nodes with wrong labels) in $\mathcal{T}'$ using random initialization with 2 models and under different threshold $\tau_c$ (adjusting $\tau_c$ controls the number of added nodes) in Figure~\ref{fig:diversity}. The results show that random initialization does not significantly outperform a single model. We found that this is because the two models lack diversity. For example, two randomly initialized models provide the same label prediction for 2,900 nodes (out of a total number of 3,327 nodes) on the \citeseer dataset and the prediction accuracy in these agreed nodes is 71.9\%. We found that this phenomenon is consistent across different GNN models and datasets. It is observed that GNN models resemble label propagation algorithm in some sense~\citep{wang2020unifying} and the results of label propagation are totally determined by the graph structure and the labeled nodes. Therefore, two GNN models trained with different random initializations tend to produce the same label prediction because they use the same graph structure and training set.

\noindent \textbf{Creating diversity in GNN models}. Generating multiple different GNN models is straightforward with \textit{random initialization}, which trains the same model with different parameter initializations. We show the number of errors (i.e., nodes with wrong labels) in $\mathcal{T}'$ using random initialization with 2 models and under different threshold $\tau_c$ (adjusting $\tau_c$ controls the number of added nodes) in Figure~\ref{fig:diversity}. The results show that random initialization does not significantly outperform a single model. We found that this is because the 2 models lack diversity. To be more specific, 2 randomly initialized models provide the same label prediction for 2,900 nodes (3,327 nodes in total) on the \citeseer dataset and the prediction accuracy in these agreed nodes is 71.9\%. This phenomenon is consistent across different GNN models and datasets. It is observed that GNN models resemble label propagation in some sense~\citep{wang2020unifying} and the results of label propagation are totally determined by the graph structure and the labeled nodes. Therefore, 2 randomly initialized GNN models lack diversity because they use the same graph structure and training set.

%Motivated by this finding, we propose to generate multiple GNN models with better diversity using \textit{train set swapping}, which randomly re-partitions the visible set (training and validation set, i.e., $\mathcal{T} \cup \mathcal{S}$) for each model. Train set swapping first unites the original training set $\mathcal{T}$ and validation set $\mathcal{S}$. Then $|\mathcal{T}|$ nodes in the visible set are randomly selected as the training set for a model and the remaining samples go to the validation set. The motivation is to use a different training set to train each GNN model for better diversity. We also plot the errors in the $\mathcal{T}'$ produced by train set swapping with 2 models in Figure~\ref{fig:diversity}. The results show that train set swapping generates significantly fewer errors than random initialization when adding the same number of nodes. This is because the models have better diversity than random initialization and they agree on the label prediction of only 2,230 nodes on the \citeseer dataset. The prediction accuracy in the agreed nodes is 85.4\%, which is significantly higher than the 71.9\% accuracy for random initialization. In the implementation, we also ensure \textit{class balance} for TNA, which means that each class has the same number of nodes in $\mathcal{T}'$. If the number of nodes to be added to $\mathcal{T}'$ for a class is larger than that for the smallest class, then we add only the nodes with the largest confidence for this class.

Motivated by this finding, we propose to generate multiple GNN models with better diversity using \textit{train set swapping}, which randomly re-partitions the visible set (training and validation set, i.e., $\mathcal{T} \cup \mathcal{S}$) for each model. $|\mathcal{T}|$ nodes in the visible set are randomly selected as the training set for a model and the remaining samples go to the validation set. The motivation is to use a different training set to train each GNN model for better diversity. We also plot the errors in the $\mathcal{T}'$ produced by train set swapping with 2 models in Figure~\ref{fig:diversity}. The results show that train set swapping generates significantly fewer errors than random initialization. This is because the 2 models have better diversity than random initialization. They agree on the label prediction of only 2,230 nodes on the \citeseer dataset and the prediction accuracy in the agreed nodes is 85.4\% (71.9\% for random initialization). When implementing TNA, we also ensure \textit{class balance}, which means that each class has the same number of nodes in $\mathcal{T}'$. If the number of nodes to be added to $\mathcal{T}'$ for a class is larger than that for the smallest class, we add only the nodes with the highest confidence for this class. The motivation is to avoid biasing the model to certain classes due to class imbalance.

%\paragraph{Class balance.} A trick that is crucial for the performance of TNA is ensuring that each class has a similar number of nodes in the enlarged training set $\mathcal{T}'$. We observed that different classes can have a very different number of nodes. For example, for the Coauthor CS dataset, the number of nodes in the largest class is 4.78x that of the smallest class. If we assume that every node has the same probability of being added to $\mathcal{T}'$, the large classes can have significantly more training samples than the small classes. We found that TNA can even degrade the accuracy (compared to without TNA) in this case. We conjecture that this is because an unbalanced training set encourages the GNN model to label nodes as from the large classes, which does not generalize. Therefore, we constrain each class to have the same number of nodes in $\mathcal{T}'$. If the number of nodes to be added to $\mathcal{T}'$ for a class is larger than that for the smallest class, then we add only the nodes with the largest confidence for this class.

%\section{The Relation between SEG and Other Methods}

% !TEX root = ./main.tex

% \vspace{-4mm}
\section{Experimental Results}\label{sec:experiment}

\begin{table}[t]
	% \vspace{-3mm}
	\caption{Dataset statistics}
	\label{table:dataset_statistics}
	% \fontsize{7}{8}\selectfont
	\centering
	% \resizebox{0.99\linewidth}{!}{%
	\begin{tabular}{l r r r r r r}
		\toprule
		& \textbf{Classes} & \textbf{Features} & \textbf{Nodes} & \textbf{Edges}  & $\mathbf{\alpha}\,\,$ \\
		\midrule
		\textbf{\cora}         &            7 &          1,433 &       2,708 &       5,278 & 0.19  &    \\
		\textbf{\citeseer}     &            6 &          3,327 &       4,552 &       3,668 & 0.26  &    \\
		\textbf{\pubmed}       &            3 &           500 &      19,717 &      44,324  & 0.19  &   \\
		\textbf{CS}        &           15 &          6,805 &      18,333 &      81,894 & 0.19  &    \\
		\textbf{Physics}        &            5 &          8,415 &      34,493 &     247,962 & 0.06  &    \\
		\textbf{Computers}       &           10 &           767 &      13,752 &     245,861 & 0.22  &    \\
		\textbf{Photo}      &            8 &           745 &       7,650 &     119,081  & 0.17  &   \\
		\bottomrule
	\end{tabular}
	% }
\end{table}

% \vspace{-1mm}
\subsection{Experiment settings}

Seven popular GNN benchmark datasets were used in the experiments and we list the satistics and \textit{noise ratio} $\alpha$ of them in Table~\ref{table:dataset_statistics}. Among them, \cora, \citeseer and \pubmed are 3 well known citation networks  and we used the version provided by \cite{yang2016revisiting}. \amcomp and \amphoto are derived from the Amazon co-purchase graph in \citep{mcauley2015image}. \magcs and \magph are obtained from the Microsoft Academic Graph for the KDD Cup 2016 challenge\footnote{https://www.kdd.org/kdd-cup/view/kdd-cup-2016}. For these 4 datasets, we used the version pre-processed by \cite{shchur2019pitfalls}.  We evaluated our methods on 3 popular GNN models, i.e., GCN~\citep{kipf2017semi}, GAT~\citep{velickovic2018graph} and SGC~\citep{felix2019simplifying}. All the three models are configured to have two layers because GNN models usually perform the best with two layers~\citep{oono2020graph}. Weights in the models were initialized according to~\citep{pmlr-glorot10a} and the biases were initialized as zeros. The models were trained using the Adam optimizer~\citep{kingma2014adam} and the learning rate was set to 0.01. For both TU and TNA, we utilized a grid search to tune their parameters (i.e., the thresholds $\tau_d$, $\tau_a$ and $\tau_c$) according to classification accuracy on \textit{the validation set}. 

% Our GCN model implementation has 2 GCN convolutional layers with a hidden size of 16. The activation function is \textit{ReLU}. A dropout layer with a dropout rate of 0.5 is used after the first GCN layer. Our GAT model implementation has 2 GAT layers with an attention coefficient dropout probability of 0.6. The first layer is an 8-heads attention layer with a hidden size of 8. The second layer has a hidden size of $8 \times 8$. The activation function is \textit{ELU} . Two dropout layers with a dropout rate of 0.6 are used between the input layer and the first GAT layer, and between the first GAT layer and the second GAT layer. Our SGC model implementation has an SGC convolutional layer with 2 hops (equivalent to 2 SGC layers according to the SGC definition).  

We followed the evaluation protocol proposed by~\citep{shchur2019pitfalls} and recorded the average classification accuracy and standard deviation of 10 different dataset splits. For each split, 20 and 30 nodes from each class were randomly sampled as the training set and validation set, respectively, and the other nodes were used as the test set. Under each split, we ran 10 random initializations of the model parameters and used the average accuracy of the 10 initializations as the performance of this split. This evaluation protocol excludes the influence of data split on the performance, which was found to be significant. All the models and algorithms in the experiments are implemented on PyTorch~\cite{pytorch} and PyTorch-Geometric~\cite{pytorch-geometric}.\footnote{Code is available at \url{https://github.com/yang-han/Self-Enhanced_GNN} .} More implementation details can be found in Section D of the supplementary material.

%We evaluated our methods on 3 popular GNN models, i.e., GCN~\citep{kipf2017semi}, GAT~\citep{velickovic2018graph} and SGC~\citep{felix2019simplifying}. We used 7 datasets to evaluate our methods. Among them, \cora, \citeseer and \pubmed are 3 well known citation networks  and we used the version provided by \citet{yang2016revisiting}. \amcomp and \amphoto are derived from the Amazon co-purchase graph in \citet{mcauley2015image}. \magcs and \magph are obtained from the Microsoft Academic Graph for the KDD Cup 2016 challenge~\footnote{https://www.kdd.org/kdd-cup/view/kdd-cup-2016}. For these 4 datasets, we used the version pre-processed by \citet{shchur2019pitfalls}. The statistics of the datasets are summarized in Table~\ref{table:dataset_statistics}, where $\mathbf{\,\alpha}$ is the \textit{noise ratio} defined in Section~\ref{sec:topology update}.

% and the code to reproduce the results is available anonymously at xxx

\begin{table*}[t]
	\caption{Performance results of self-enhanced GNN (SEG), where \textit{Error Reduction} is the percentage of classification error reduced from the respective baseline model}
	\label{tab:overall}
	\centering
% \resizebox{0.95\linewidth}{!}{%
	\begin{tabular}{ lccccccc }
		\toprule
		& \textbf{\cora}& \textbf{\citeseer}& \textbf{\pubmed}  & \textbf{CS}   & \textbf{Physics}   & \textbf{Computers}  &\textbf{Photo}\\
\midrule
		\textbf{GCN}                            & 78.7\PM{1.5}  & 66.5\PM{2.4}  & 75.5\PM{1.8}  & 90.7\PM{0.6}  & 93.1\PM{0.5}  & 71.9\PM{12.8}  & 85.2\PM{10.0}   \\
		\textbf{GCN+SEG} 			    & 82.3\PM{1.2}	&	71.1\PM{0.8}	&	80.0\PM{1.4}	&	92.9\PM{0.4}	&	93.9\PM{0.2}	&	80.2\PM{6.5}	&	90.4\PM{0.9} \\
		\textbf{Error Reduction}     &\BF{16.9\%}	&	\BF{13.7\%}	&	\BF{18.4\%}	&	\BF{23.7\%}	&	\BF{11.6\%}	&	\BF{29.5\%}	&	\BF{35.1\%} \\
		
		\midrule
		\textbf{GAT}                                        & 79.0\PM{1.7}  & 65.7\PM{1.9}  & 75.3\PM{2.4}  & 89.9\PM{0.6}  & 92.0\PM{0.8}  & 82.2\PM{2.1}  & 89.6\PM{1.8}      \\
		\textbf{GAT+SEG} &81.4\PM{1.3}	&	70.0\PM{1.0}	&	78.9\PM{1.4}	&	91.6\PM{0.5}	&	93.5\PM{0.4}	&	83.7\PM{0.7}	&	90.8\PM{1.4} \\
		\textbf{Error Reduction}& \BF{11.4\%}	&	\BF{12.3\%}	&	\BF{14.6\%}	&	\BF{16.8\%}	&	\BF{18.8\%}	&	\BF{8.4\%}	&	\BF{11.5\%} \\
		
		\midrule
		\textbf{SGC}                                          & 77.4\PM{2.6}  & 65.0\PM{2.0}  & 73.3\PM{2.6}  & 91.3\PM{0.6}  & 93.3\PM{0.3}  & 81.1\PM{2.0}  & 89.3\PM{1.4}     \\
		\textbf{SGC+SEG} & 82.2\PM{1.3}	&	70.2\PM{0.9}	&	78.1\PM{2.3}	&	93.1\PM{0.2}	&	94.1\PM{0.4}	&	82.8\PM{1.7}	&	89.9\PM{0.8} \\
		\textbf{Error Reduction} & \BF{21.2\%}	&	\BF{14.9\%}	&	\BF{8.5\%}	&	\BF{16.8\%}	&	\BF{18.8\%}	&	\BF{9.0\%}	&	\BF{7.7\%} \\
		\bottomrule
	\end{tabular}
	% }
	% \vspace{-2mm}
\end{table*}

%\subsection{Experiment Settings}
% \vspace{1mm}

%\paragraph{Settings and evaluation methodology.}  We configured all the three models to have two layers because GNN models usually perform the best with two layers due to over-smoothing~\citep{oono2020graph} and increasing the layers also increases the computation cost exponentially due to neighbor propagation. All weights for the models were initialized according to~\citep{pmlr-glorot10a} and all biases were initialized as zeros. The models were trained using the Adam~\citep{kingma2014adam} optimizer and the learning rate was set to 0.01. For both TU and TNA, we utilized a grid search to tune their parameters (i.e., the thresholds $\tau_d$, $\tau_a$ and $\tau_c$) \textit{on the validation set}. 

%We followed the evaluation protocol proposed by \citet{shchur2019pitfalls} and recorded the average classification accuracy and standard deviation of 10 different dataset splits. For each split, 20 and 30 nodes from each class were randomly sampled as the training set and validation set, respectively, and the other nodes were used as the test set. Under each split, we ran 10 random initializations of the model parameters and used the average accuracy of the 10 initializations as the performance of this split. The motivation of this evaluation protocol was to exclude the influence of the randomness in data split on performance, which was found to be  significant.

\subsection{Overall Performance Results}

We first report the overall performance results of self-enhanced GNN (\textbf{SEG}) in Table~\ref{tab:overall}. The reported performance of SEG is the best performance that can be obtained using TU, TNA, or (TU + TNA). In practice, we may choose to use TU, TNA, or (TU + TNA) by their prediction accuracy on the validation set. The results in Table~\ref{tab:overall} show that SEG consistently improves the performance of the 3 GNN models on the 7 datasets, where the reduction in classification error is 16.2\% on average and can be as high as 35.1\%. The result is significant particularly because it shows that SEG is an effective, general framework that improves the performance of well-known GNN models that are already recognized to be effective. In the subsequent subsections, we analyze the performance of TU and TNA individually, as well as examine how they influence data quality.

\begin{table*}[t]
	\caption{Performance results of TU, \textit{Delete} refers to edge deletion, \textit{Add} refers to edge addition, and \textit{Modify} refers to conducting both edge deletion and addition, best-performing variant for each model marked in bold}
	\label{table:topology update}
	\centering
	% \resizebox{0.95\linewidth}{!}{%
		\begin{tabular}{ @{}lccccccc@{}}
			\toprule
			& \textbf{\cora}& \textbf{\citeseer}& \textbf{\pubmed}  & \textbf{CS}   & \textbf{Physics}   & \textbf{Computers}  &\textbf{Photo}\\
			\midrule
			%		\textbf{GCN Baseline}                           & 78.7\PM{1.5}  & 66.5\PM{2.4}  & 75.5\PM{1.8}  & 90.7\PM{0.6}  & 93.1\PM{0.5}  & 71.9\PM{12.8}  & 85.2\PM{10.0}   \\
			%		\textbf{GAT Baseline}                            & 79.0\PM{1.7}  & 65.7\PM{1.9}  & 75.3\PM{2.4}  & 89.9\PM{0.6}  & 92.0\PM{0.6}  & 82.2\PM{2.1}  & 89.6\PM{1.8}  \\
			%		\textbf{SGC Baseline}                            & 77.4\PM{2.6}  & 65.0\PM{2.0}  & 73.3\PM{2.6}  & 91.3\PM{0.6}  & 93.3\PM{0.3}  & 81.1\PM{2.0}  & 89.3\PM{1.4}   \\
			%		
			%		\midrule\midrule
			\textbf{GCN}                            & 78.7\PM{1.5}  & 66.5\PM{2.4}  & 75.5\PM{1.8}  & 90.7\PM{0.6}  & 93.1\PM{0.5}  & 71.9\PM{12.8}  & 85.2\PM{10.0}   \\
			\textbf{GCN+Delete}               & 79.2\PM{1.6}  & \underline{66.5\PM{2.4}}  & 75.6\PM{2.0}  & \textbf{91.8\PM{0.6}}  & 93.2\PM{0.6}  & \textbf{80.1\PM{2.1}}  & \textbf{89.0\PM{2.5}}     \\
			\textbf{GCN+Add}                 & 78.8\PM{1.7}  & 66.8\PM{2.4}  & 75.6\PM{1.7}  & \underline{90.7\PM{0.7}}  & 93.2\PM{0.4}  & 78.9\PM{2.2}  & 88.2\PM{2.3}    \\
			\textbf{GCN+Modify}                & \textbf{79.4\PM{1.3}}  & \textbf{67.1\PM{2.2}}  & \textbf{75.9\PM{2.0}}  & 91.7\PM{0.9}  & \textbf{93.4\PM{0.3}}  & 79.2\PM{2.5}  & 88.5\PM{4.0}   \\
			\midrule
			\textbf{GAT}                                        & 79.0\PM{1.7}  & 65.7\PM{1.9}  & 75.3\PM{2.4}  & 89.9\PM{0.6}  & 92.0\PM{0.8}  & 82.2\PM{2.1}  & 89.6\PM{1.8}      \\
			\textbf{GAT+Delete}                & \textbf{79.3\PM{1.8}}  & \textbf{65.8\PM{1.9}}  & \underline{75.3\PM{2.6}}  & \textbf{90.9\PM{0.9}}  & \textbf{92.2\PM{0.7}}  & \textbf{82.8\PM{2.1}}  & \textbf{90.3\PM{1.5}}  \\
			\textbf{GAT+Add}                 & 79.1\PM{1.3}  & \underline{65.7\PM{2.0}}  & 75.7\PM{1.8}  & 90.0\PM{0.5}  & 92.1\PM{0.8}  & 82.6\PM{2.5}  & 89.7\PM{0.8}  \\
			\textbf{GAT+Modify}                & 79.1\PM{1.8}  & 65.8\PM{2.1}  & \textbf{76.0\PM{2.2}}  & 90.7\PM{0.9}  & 92.1\PM{0.9}  & 82.4\PM{2.0}  & 90.1\PM{1.4}   \\
			\midrule
			\textbf{SGC}                                          & 77.4\PM{2.6}  & 65.0\PM{2.0}  & 73.3\PM{2.6}  & 91.3\PM{0.6}  & 93.3\PM{0.3}  & 81.1\PM{2.0}  & 89.3\PM{1.4}     \\
			\textbf{SGC+Delete}                & 77.8\PM{2.1}  & 65.5\PM{2.4}  & 73.6\PM{2.7}  & 92.6\PM{0.4}  & 93.5\PM{0.4}  & \textbf{82.0\PM{2.0}}  & \textbf{89.6\PM{1.4}} \\
			\textbf{SGC+Add}                 & 77.5\PM{2.4}  & 65.7\PM{1.7}  & 73.8\PM{2.5}  & 91.5\PM{0.6}  & 93.5\PM{0.4}  & 81.6\PM{1.9}  & 89.4\PM{1.5}  \\
			\textbf{SGC+Modify}                 & \textbf{78.5\PM{2.3}}  & \textbf{66.7\PM{1.6}}  & \textbf{74.0\PM{2.6}}  & \textbf{92.7\PM{0.3}}  & \textbf{93.5\PM{0.3}}  & 81.7\PM{2.2}  & 89.4\PM{1.6}  \\
			\bottomrule
		\end{tabular}
	% }
	% \vspace{-3mm}
\end{table*}

\subsection{Results for Topology Update}
\begin{table}[t]
	\caption{The effect of edge deletion and edge addition on noise ratio for \cora. For edge deletion, the reported tuple is the number of deleted inter-class edges and intra-class edges, respectively. For edge addition, the reported tuple is the number of added intra-class edges and inter-class edges, respectively. The noise ratio $\alpha$ of the original graph is 19.00\%.}
	\label{tab:noise ratio analysis}
	\centering
	% \begin{center}
		% \begin{sf}
			% \fontsize{8}{9}\selectfont
	\resizebox{\linewidth}{!}{
			\begin{tabular}{ccccccc}
				\toprule
				Model & Edge Deletion & $\alpha$ after Deletion & Edge Addition & $\alpha$ after Addition \\
				\midrule
				GCN & (332, 218) & 14.19\% & (4692, 85) & 10.82\% \\
				\midrule
				GAT & (309, 212) & 14.59\% & (5995, 165) & 10.21\% \\
				\midrule
				SGC &  (242, 116) & 15.47\% & (3807, 25) & 11.28\% \\
				\bottomrule
			\end{tabular}
	}
		% \end{sf}
	% \end{center}
	% \vspace{-4mm}
\end{table}

\begin{figure}[t]	
	\centering
	\begin{subfigure}{0.49\linewidth}
		\centering 
		\includegraphics[width=\textwidth]{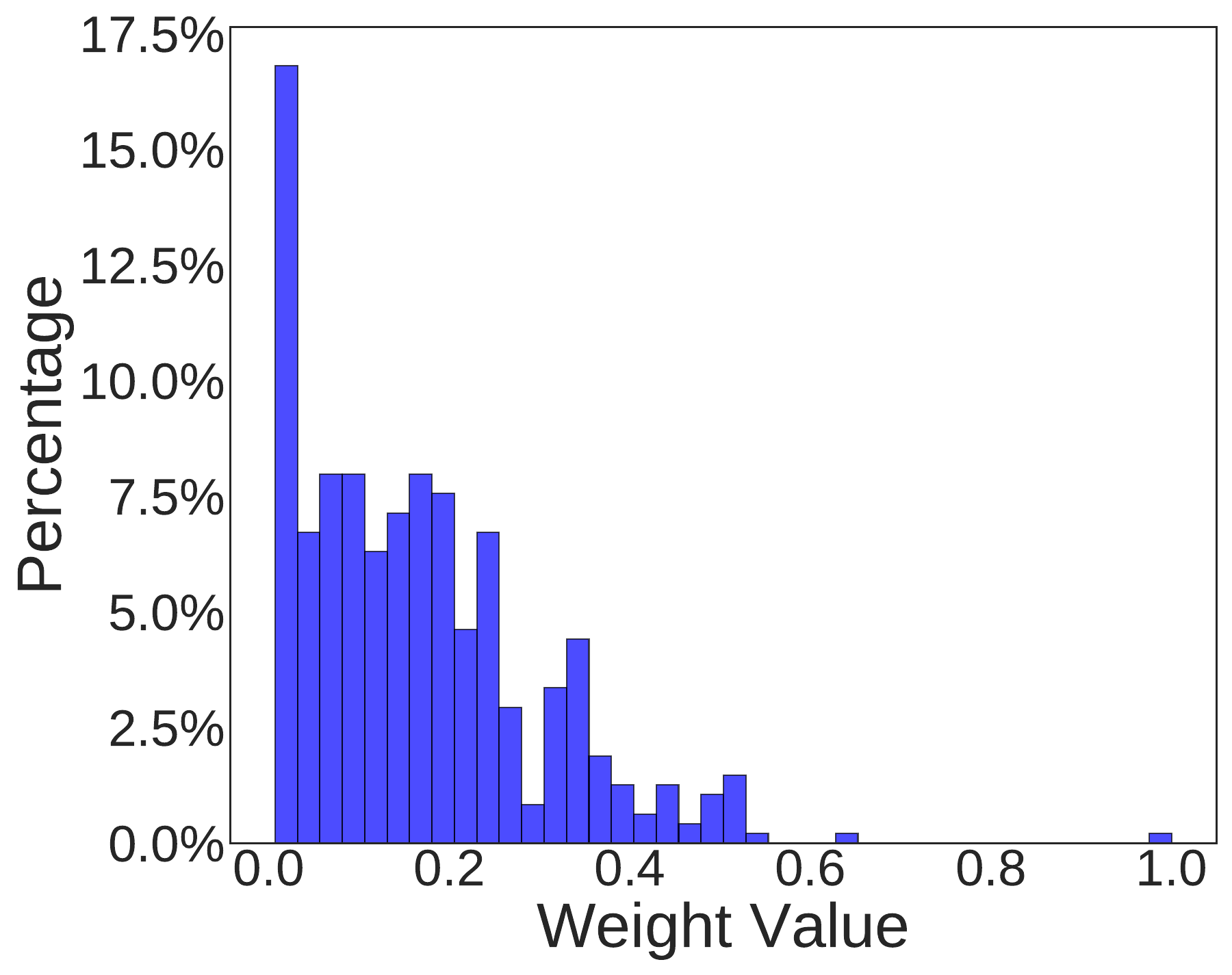} 
		\caption{Deleted edges (mean=0.17)}
		\label{ig:weight of deleted edges}
	\end{subfigure}
	\begin{subfigure}{0.49\linewidth}
		\centering 
		\includegraphics[width=\textwidth]{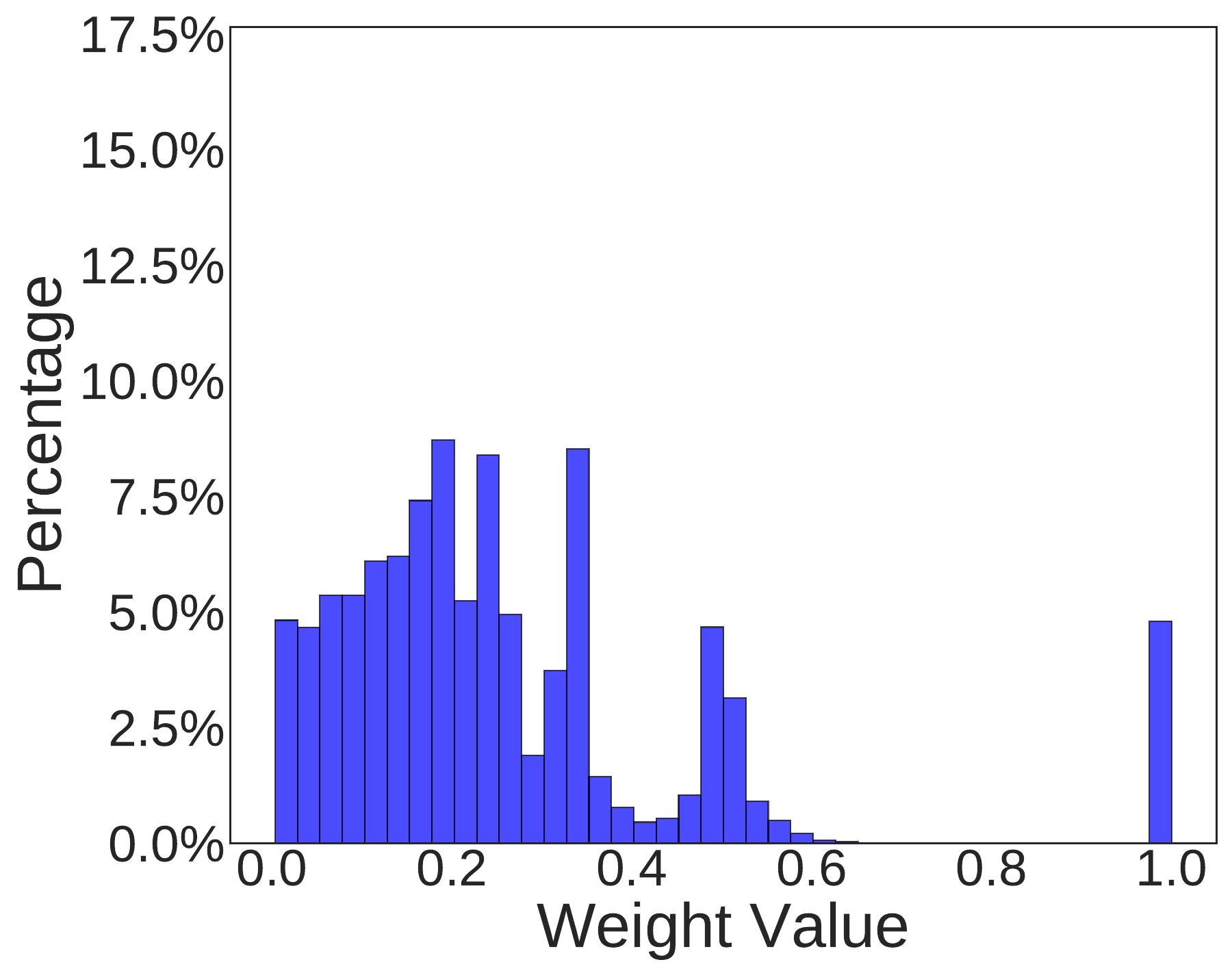} 
		\caption{Kept edges (mean=0.26)}
		\label{fig:weight of kept edges}
	\end{subfigure}
	\caption{The distribution of GAT attention weights on the edges that are deleted and kept by Algorithm~\ref{alg:edge deletion} on  \cora} 
	\label{fig:correlation with GAT}
	% \vspace{-4mm}
\end{figure}

The performance results for three variants of TU (i.e., \textit{Delete}, \textit{Add} and \textit{Modify}) are reported in Table~\ref{table:topology update}. To control the complexity of parameter search, we constrained the number of added edges to be the same as deleted edges for \textit{Modify}. We make the following observations from the results in Table~\ref{table:topology update}.

\begin{table*}[t]
	\caption{Performance results of TNA, best-performing variant for each model marked in bold}
\label{table:TNA}
\centering
% \resizebox{0.95\linewidth}{!}{%
\begin{tabular}{ lccccccc }
\toprule
& \textbf{\cora}& \textbf{\citeseer}& \textbf{\pubmed}  & \textbf{CS}   & \textbf{Physics}   & \textbf{Computers}  &\textbf{Photo}\\
\midrule
%		\textbf{GCN Baseline}                           & 78.7\PM{1.5}  & 66.5\PM{2.4}  & 75.5\PM{1.8}  & 90.7\PM{0.6}  & 93.1\PM{0.5}  & 71.9\PM{12.8}  & 85.2\PM{10.0}   \\
%		\textbf{GAT Baseline}                            & 79.0\PM{1.7}  & 65.7\PM{1.9}  & 75.3\PM{2.4}  & 89.9\PM{0.6}  & 92.0\PM{0.6}  & 82.2\PM{2.1}  & 89.6\PM{1.8}  \\
%		\textbf{SGC Baseline}                            & 77.4\PM{2.6}  & 65.0\PM{2.0}  & 73.3\PM{2.6}  & 91.3\PM{0.6}  & 93.3\PM{0.3}  & 81.1\PM{2.0}  & 89.3\PM{1.4}   \\
%		\midrule
\textbf{GCN}                            & 78.7\PM{1.5}  & 66.5\PM{2.4}  & 75.5\PM{1.8}  & 90.7\PM{0.6}  & 93.1\PM{0.5}  & 71.9\PM{12.8}  & 85.2\PM{10.0}   \\
\textbf{GCN+TNA}                 & \textbf{82.1\PM{1.1}}  & \textbf{70.6\PM{1.1}}  & \textbf{80.0\PM{1.4}}  & \textbf{91.8\PM{0.3}}  & \textbf{93.7\PM{0.5}}  & \textbf{80.2\PM{6.5}}  & \textbf{89.5\PM{2.6}}  \\
\textbf{Ensemble} & 78.8\PM{1.0}  & 67.2\PM{1.9}  & 76.1\PM{1.9}  & 90.9\PM{0.6}  & 92.8\PM{0.3}  & 79.5\PM{2.8}  & 88.1\PM{1.9} \\
\textbf{Distillation}  & 80.8\PM{2.9}  & 63.9\PM{2.5}  & 79.2\PM{2.6}  & 90.3\PM{0.8}  & 92.9\PM{2.0}  & 75.1\PM{15.5}  & 84.3\PM{12.0} \\
% \textbf{Error Reduction}         & \BF{16.0\%} & \BF{12.2\%} & \BF{18.4\%} & \BF{11.8\%} & \BF{8.7\%} & \BF{29.5\%} & \BF{29.1\%}    \\
\midrule
\textbf{GAT}                                        & 79.0\PM{1.7}  & 65.7\PM{1.9}  & 75.3\PM{2.4}  & 89.9\PM{0.6}  & 92.0\PM{0.8}  & 82.2\PM{2.1}  & 89.6\PM{1.8}      \\
\textbf{GAT+TNA}                  & \textbf{81.4\PM{1.3}}  & \textbf{70.0\PM{1.0}}  & 78.9\PM{1.4}  & \textbf{91.1\PM{0.4}}  & \textbf{93.4\PM{0.3}}  & 82.7\PM{1.7}  & \textbf{90.8\PM{1.4}}  \\		
\textbf{Ensemble}  & 77.1\PM{1.3}  & 65.2\PM{2.0}  & 75.0\PM{1.8}  & 89.8\PM{0.8}  & 92.0\PM{1.0}  & 80.9\PM{2.1}  & 89.4\PM{1.9} \\
\textbf{Distillation} & 80.5\PM{2.7}  & 63.0\PM{3.1}  & \textbf{79.4\PM{3.9}}  & 89.0\PM{0.8}  & 91.9\PM{2.1}  & \textbf{87.7\PM{1.6}}  & 90.5\PM{1.3} \\
% \textbf{Error Reduction}   	   & \BF{11.4\%} & \BF{12.3\%} & \BF{14.6\%} & \BF{11.9\%} & \BF{5.7\%} & \BF{2.8\%} & \BF{11.5\%}    \\
\midrule
\textbf{SGC}                                          & 77.4\PM{2.6}  & 65.0\PM{2.0}  & 73.3\PM{2.6}  & 91.3\PM{0.6}  & 93.3\PM{0.3}  & 81.1\PM{2.0}  & 89.3\PM{1.4}     \\
\textbf{SGC+TNA}                & \textbf{82.2\PM{1.3}}  & \textbf{70.2\PM{0.9}}  & 73.3\PM{3.2}  & \textbf{92.0\PM{0.4}}  & \textbf{93.9\PM{0.3}}  & 82.8\PM{1.7}  & 89.9\PM{1.5}    \\
\textbf{Ensemble}  & 79.2\PM{1.3}  & 66.6\PM{1.5}  & 75.1\PM{1.8}  & 91.5\PM{0.5}  & 93.2\PM{0.2}  & 81.4\PM{1.7}  & 89.1\PM{1.7} \\
\textbf{Distillation}  & 80.9\PM{2.9}  & 63.0\PM{2.5}  & \textbf{76.3\PM{4.1}}  & 91.0\PM{0.7}  & 92.5\PM{2.2}  & \textbf{86.7\PM{2.9}}  & \textbf{90.0\PM{1.8}} \\
% \textbf{Error Reduction}        & \BF{21.2\%} & \BF{14.9\%} & \BF{0.0\%} & \BF{8.0\%} & \BF{9.0\%} & \BF{9.0\%} & \BF{5.6\%}      \\
\bottomrule
\end{tabular}
% }
\end{table*}
% the resultant graph has the same number of edges as the initial graph

%Edge deletion can be disabled by setting $\tau_d=0$ while edge addition can be disabled by setting $\tau_a=1$. 

%. Even if the probability of adding an inter-class edge is small (the same as the probability of keeping an inter-class edge in $\mathcal{G}$ in edge deletion), the algorithm may still add a considerable number of inter-class edges in expectation. 

Firstly, TU improves the performance of GCN, GAT and SGC in most cases and the improvement is significant in some cases. For example, the accuracy increases by 8.2\% for GCN on the Amazon Computers dataset. There is no improvement in 4 out of the 63 cases (underlined) because threshold-based tuning (for $\tau_d$ and $\tau_a$) on the validation set rejects TU as it cannot improve the performance. 
Secondly, edge deletion generally achieves greater performance improvements than edge addition. This is because there is a large number of possible inter-class edges (e.g., $n^2/c$ if the classes are balanced) and \textit{Add} may introduce a considerable number of inter-class edges in expectation. 
Thirdly, the performance improvement of TU is smaller for CiteSeer and PubMed than for the other datasets. This is because the accuracy of all three models is considerably lower for CiteSeer and PubMed than for the other datasets. Thus, the TU algorithms are more likely to make wrong decisions (i.e., deleting intra-class edges or adding inter-class edges) as TU decisions are guided by the model predictions. Fourth, although the baseline models have high accuracy for both Coauthor CS and Coauthor Physics, TU has considerably greater performance improvements on Coauthor CS than on Coauthor Physics. This can be explained as the noise ratio of the original Coauthor Physics graph is much lower than the Coauthor CS graph  (6\% vs. 19\%), and thus reducing noise ratio has smaller influence on the performance for Coauthor Physics.

Detailed profiling finds that the attention weights of GAT perform a role similar to edge deletion, which we illustrate in Figure~\ref{fig:correlation with GAT}. The results show that the GAT attention weights on the deleted edges are significantly smaller than the kept edges, which reduces the influence of the inter-class edges. This suggests that GAT should be less sensitive to changes in noise ratio. However, GAT cannot really set the weights of the inter-class edges to 0 as it uses the \textit{softmax} function to compute attention weights. In contrast, edge deletion removes inter-class edges and thus improves GAT in most cases. 

We also examined the edge deletion and addition decisions made by TU in Table~\ref{tab:noise ratio analysis}. For both edge deletion and addition, we report the number of correct decisions (i.e., removing inter-class edges for deletion and adding intra-class edges for addition) and wrong decisions (i.e., removing intra-class edges for deletion and adding inter-class edges for addition), and the noise ratio of the \cora graph after TU. The results show that TU effectively reduces noise ratio. Most of the added edges are intra-class edges and only a few are inter-class edges. Edge deletion effectively removes inter-class edges but a considerable number of intra-class edges are also removed. This is because there are much more intra-class edges in the graph than inter-class edges. The probability of removing an intra-class edge is actually low.

\subsection{Results for Training Node Augmentation}

\begin{table}[t]
	\caption{The number of nodes added into $\mathcal{T}'$ by TNA and the number of errors (nodes with wrong label) in these added nodes for the \cora dataset}
	\label{tab:node addition}
	\centering
	% \begin{center}
	% \begin{sf}
		% \fontsize{8}{9}\selectfont
		\begin{tabular}{cccc}
			\toprule
			Model & GCN & GAT & SGC  \\
			\midrule
			\# Added Nodes & 826      & 714 & 637  \\
			\midrule
			\# Errors  & 83 & 72 & 45 \\
			\midrule
			Error Ratio & 10.05\% & 10.08\% & 7.06\%\\  
			\bottomrule
		\end{tabular}
	% \end{sf}
	% \end{center}
\end{table}

We present the performance results of TNA (using 2 models) in Table~\ref{table:TNA}. As TNA uses pre-trained models, we introduce two additional baselines that also use pre-trained models. \textit{Ensemble}~\citep{dietterich2000ensemble} averages the outputs of two models to make classification decision while \textit{Distillation}~\citep{hinton2015distilling} generates soft labels for the train set and retrains the model to fit the soft labels. The results show that TNA improves the performance of the baseline original models in 20 out of the 21 cases (except SGC on PubMed). The accuracy improvements are significant in many cases, e.g., 4.3\% for GCN on the \amphoto dataset.  The performance improvements are large on \cora and \citeseer for all three GNN models. We conjecture that this is because the two datasets are relatively smaller and thus adding more training samples has a large impact on the performance.   When compared with \textit{Ensemble} and \textit{Distillation}, TNA also performs better in most cases.   Looking into the good performance of TNA, we examined the number of added nodes and the errors made by TNA in $\mathcal{T}'$ in Table~\ref{tab:node addition}.   The results show that most of the added nodes are assigned the correct label. SGC added a smaller number of nodes are added compared with GAT and GCN, and the error ratio is also lower. This may be because the model of SGC is simpler than GAT and GCN (without nonlinearity) and thus SGC is more sensitive to noise in the training samples.

\begin{figure}[t]	
	\centering
	\begin{subfigure}{0.49\linewidth}
		\centering 
		\includegraphics[width=\textwidth]{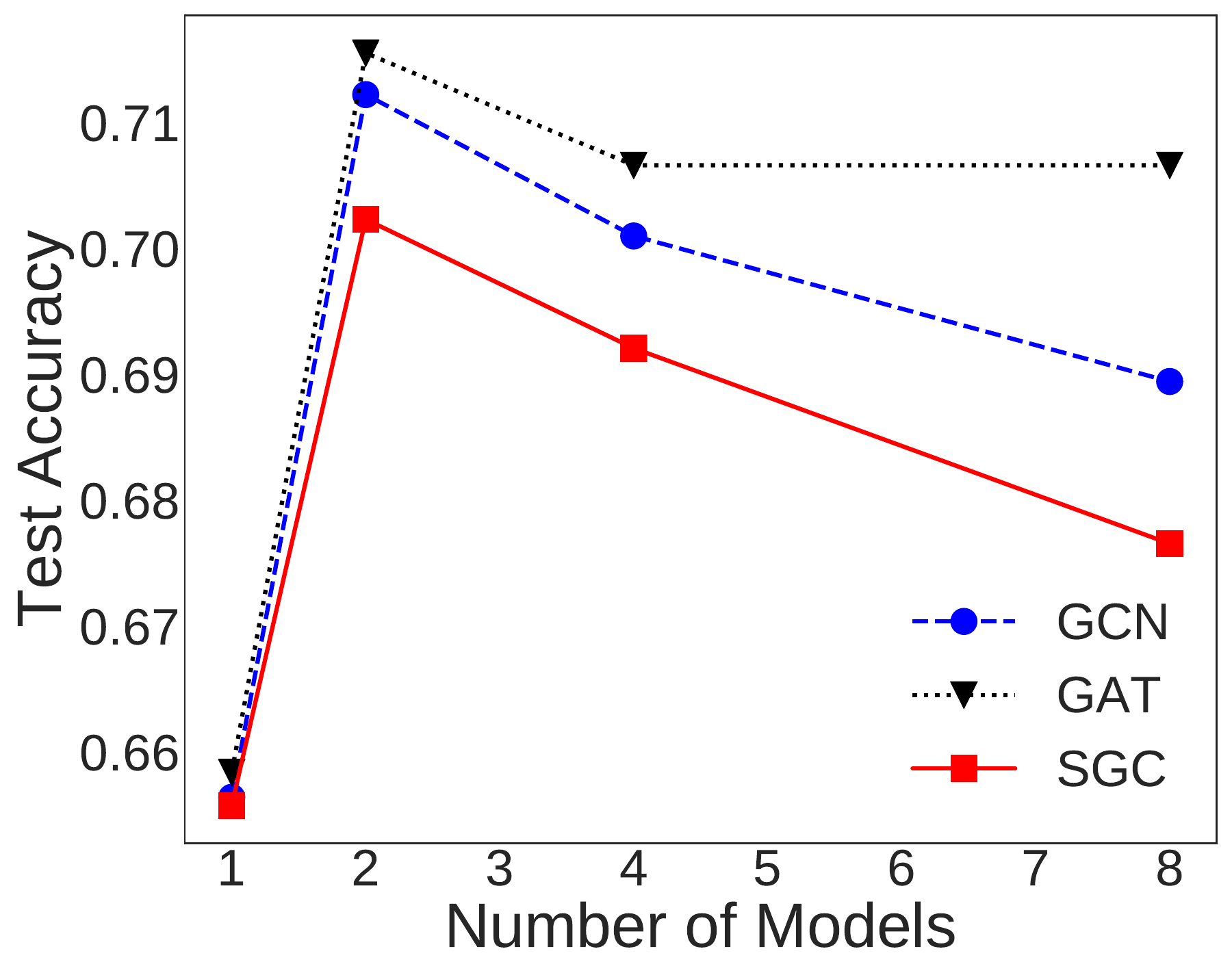} 
		\caption{Number of models}
		\label{fig:number of models}
	\end{subfigure}
	\begin{subfigure}{0.49\linewidth}
	\centering 
	\includegraphics[width=\textwidth]{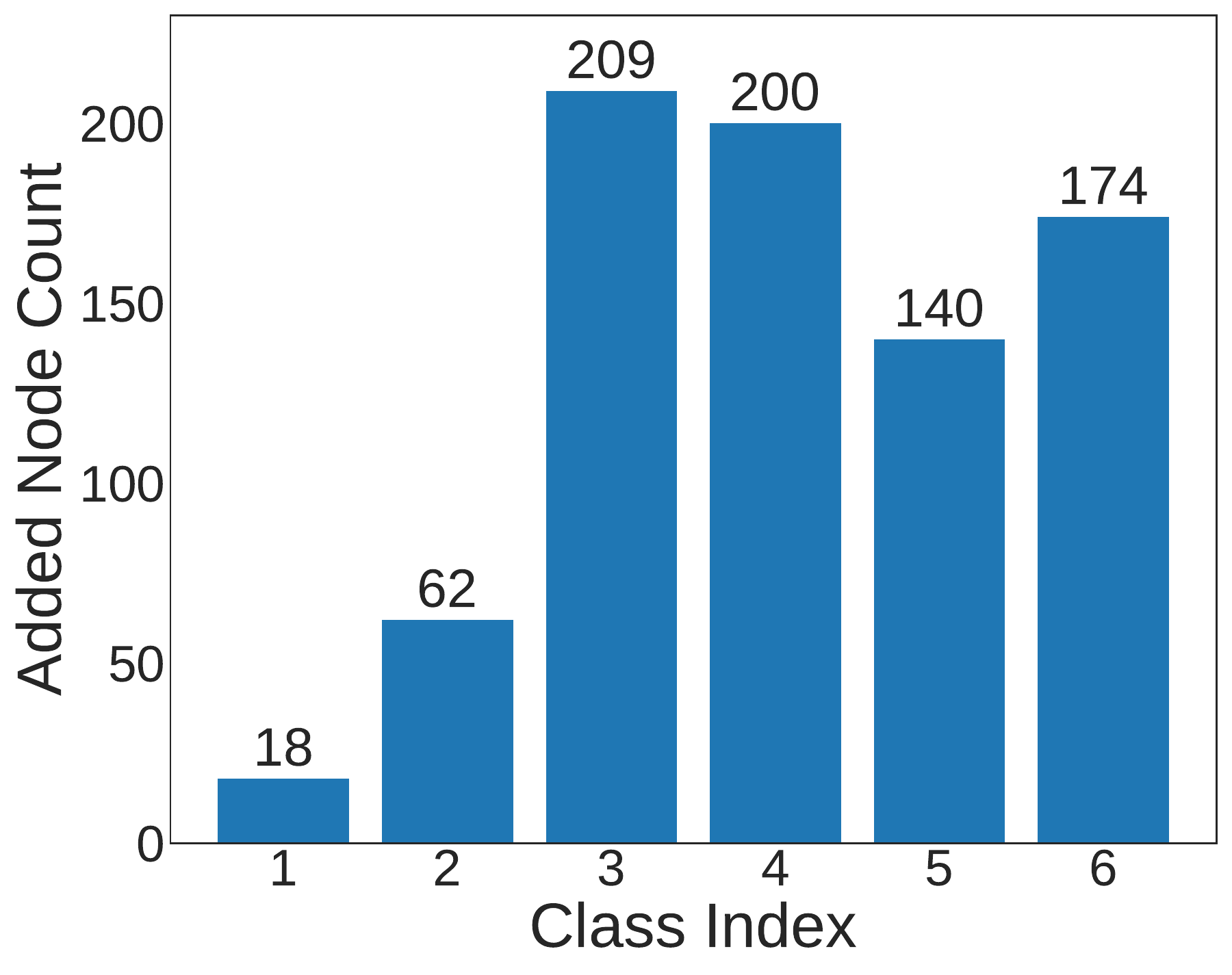} 
	\caption{Without class balance}
	\label{fig:without class balance}
	\end{subfigure}
	\caption{Examination of the designs in TNA on \citeseer} 
	\label{fig:training node augmentation}
	% \vspace{-4mm}
\end{figure}

%We examined the two important designs in TNA, i.e., multi-model diversity and class balance.

To demonstrate the benefits of using the diversity of multiple models in TNA, we report the relation between the test accuracy and the number of models (used for node selection) on the CiteSeer dataset in Figure~\ref{fig:number of models}. The result show that using 2 models provides a significant improvement in classification accuracy over 1 model, but the improvement drops when using more models. This is because more models are difficult to agree with each other and thus a low confidence threshold (i.e., $\tau_c$) needs to be used to add a good number of nodes. However, a low confidence threshold means that the added nodes are likely to contain errors. We also experimented with a version of TNA \textit{without class balance} for GCN on the Amazon Photo dataset, which records a classification accuracy of 86.6\%. In contrast, the classification accuracy with class balance is 89.5\% in Table~\ref{table:TNA}. We plot in Figure~\ref{fig:without class balance} the class distribution of the nodes added by TNA without class balance, which shows that the number of nodes in the largest class is 11.6 times of the smallest class. Thus, without class balance, the enlarged training set can be highly screwed, which leads to the model to favor certain classes.

% in Figure~\ref{fig:training node augmentation}

%\input{6_related_work}
% !TEX root = ./main.tex

\section{Conclusions}

We presented self-enhanced GNN as a general framework for co-training and self-training of GNNs to improve the quality of the input data using the outputs of existing GNN models. Two algorithms were developed in this framework, i.e., \textit{topology update}, which tries to reduce the noise ratio in the graph edges, and \textit{training node augmentation}, which enlarges the training set using pseudo labels. Many practical designs are explored and adopted in  our SEG framework, and theoretical analyses were provided to motivate and support the designs. Our experiments validated that SEG is an effective framework that consistently improves the performance of existing GNN models. We believe the SEG framework can inspire more research to pay attention to the \textit{data quality} in GNN,  and develop more algorithms for updating the graph topology and generating reliable pseudo labels for nodes to help self-training and co-training of GNNs.

\bibliographystyle{IEEEtran}
{\small
\bibliography{han}
}

% !TEX root = ./main.tex
% \appendix

\onecolumn
\appendix
% \hrule height 4pt
% \vskip 0.25in
% \vskip -\parskip%

{\centering\huge\bf Supplementary Materials for Self-Enhanced GNN\par}

% \vskip 0.29in
% \vskip -\parskip
% \hrule height 1pt
% \vskip 0.09in%	
% \twocolumn

\section{Proofs of the Theorems}

\subsection{Theorem~\ref{th:deletion} in Section~\ref{sec:topology update}}

\begin{proof}
	The probability that an intra-class edge in $\mathcal{G}$ is kept in $\mathcal{G}'$ by Algorithm~\ref{alg:edge deletion} is $p_a=\mathbb{P}\left[f(v)=f(u)|l(u)=l(v)\right]=p^2+\frac{(1-p)^2}{c-1}$. Therefore, $m_{a}=(1-\alpha) m \left(p^2+\frac{(1-p)^2}{c-1}\right)$, where $m$ is the number of edges in $\mathcal{G}$. The probability that an inter-class edge is kept is
	$p_r=\mathbb{P}\left[f(v)=f(u)|l(u)\neq l(v)\right]=\frac{2p(1-p)}{c-1}+\frac{(c-2)(1-p)^2}{(c-1)^2}$, and thus $m_{r}=\alpha m\left(\frac{2p(1-p)}{c-1}+\frac{(c-2)(1-p)^2}{(c-1)^2}\right)$. We have
	\begin{equation*}
	\begin{aligned}
	\alpha_E&=\frac{\alpha m\left(\frac{2p(1-p)}{c-1}+\frac{(c-2)(1-p)^2}{(c-1)^2}\right)}{\alpha m\left(\frac{2p(1-p)}{c-1}+\frac{(c-2)(1-p)^2}{(c-1)^2}\right)+(1-\alpha) m \left(p^2+\frac{(1-p)^2}{c-1}\right)}\\
	&<\frac{\alpha(1-p^2)}{\alpha(1-p^2)+(1-\alpha)[(c-1)p^2+(1-p)^2]}.
	\end{aligned}
	\end{equation*}  
	Solving $\frac{\alpha(1-p^2)}{\alpha(1-p^2)+(1-\alpha)[(c-1)p^2+(1-p)^2]}<\alpha$ gives $(1-\alpha)p[(c+1)p-2]\ge 0$, which is satisfied when $p>\frac{2}{c+1}$.    
\end{proof}

\subsection{Theorem~\ref{th:addition} in Section~\ref{sec:topology update}}

\begin{proof}	 
	Denote the expected number of added intra-class edges as $m'_{a}$ and the expected number of added inter-class edges as $m'_{r}$. To ensure $\alpha_{E}<\alpha$, it suffices to show that $\frac{m'_{r}}{m'_{a}+m'_{r}} < \alpha$. As there are $\frac{c-1}{c} n^2$ possible inter-class edges and $\frac{1}{c} n^2$ intra-class edges in $\mathcal{V} \times \mathcal{V}$, we have
	\begin{equation*}
	\begin{aligned}
	m'_{r}&=(\frac{c-1}{c} n^2-m \alpha)p_{r}<\frac{c-1}{c} n^2 p_{r}\\
	m'_{a}&=\left[\frac{1}{c} n^2-m (1-\alpha)\right]p_{a}>\frac{1}{c} n^2 p_{a}-m,
	\end{aligned} 
	\end{equation*} 
	where $p_{r}$ and $p_{a}$ are the probability of keeping an inter-class edge and an intra-class edge in $\mathcal{G}'$, respectively. Their expressions are given in the proof of Theorem~\ref{th:deletion}. The $m \alpha$ and $m (1-\alpha)$ terms are included to exclude the overlaps between the edges in the original graph and the edges that may be added by Algorithm~\ref{alg:edge addition}. With $m=n^2 \lambda$, we have  
	\begin{equation*}
	\frac{m'_{r}}{m'_{a}+m'_{r}}<\frac{\frac{c-1}{c} n^2 p_{r}}{\frac{c-1}{c} n^2 p_{r}+\frac{1}{c} n^2 p_{a}-m}<\frac{1-p^2}{1+\frac{(1-p)^2}{c-1}-c\lambda}.
	\end{equation*} 
	Solving $\frac{1-p^2}{1+\frac{(1-p)^2}{c-1}-c\lambda}<\alpha$ gives the result. 
\end{proof}

\subsection{Theorem~\ref{th:label} in Section~\ref{sec:training node augmentation}}

\begin{proof}
	The probability that $g^2$ gives the right label can be expressed as 
	\begin{equation*}
	\begin{aligned}
	\mathbb{P}[r^2_v=l(v)]=&\mathbb{P}[r^1_v=l(v)]\cdot\mathbb{P}[r^2_v=l(v)|r^1_v=l(v)]\\
	&+\sum_{k\neq l(v)}\mathbb{P}[r^2_v=l(v)|r^1_v=k]\cdot \mathbb{P}[r^1_v=k].
	\end{aligned}
	\end{equation*} 
	We assume that $g^2$ has a classification accuracy of $p$ and solving $\mathbb{P}[r^2_v=l(v)]=p$ gives the relation between $\beta$ and $\gamma$ as $p\gamma=p\beta+\gamma-p$. We can express $q$ as
	\begin{equation*}
	\begin{aligned}  
	q&=\frac{\mathbb{P}[r^1_v=l(v), r^2_v=l(v)]}{\mathbb{P}[r^1_v=l(v), r^2_v=l(v)]+\sum_{k\neq l(v)}\mathbb{P}[r^1_v=k, r^2_v=k]}\\
	&=\frac{(c-1)p\beta}{(c-1)p\beta+(1-p)(1-\gamma)}. 
	\end{aligned}
	\end{equation*}
	Substituting $p\gamma=p\beta+\gamma-p$ into the above expression gives $q=\frac{(c-1)p\beta}{cp\beta+1-2p}$. Solving $q\ge p$ gives the following result
	\begin{equation*}
	\begin{cases}
	\beta\ge 0 \quad \quad \quad \quad \quad \text{for} \quad p\le 1-\frac{1}{c}\\
	0 \le \beta\le \frac{1-2p}{c-1-cp} \quad \text{for} \quad p> 1-\frac{1}{c}
	\end{cases}
	.
	\end{equation*}    
    It can be verified that $\frac{1-2p}{c-1-cp}\ge 1$ when $1-\frac{1}{c}<p\le 1$. Therefore, we have $q\ge p$ regardless of the value of $p$ and $\beta$, which proves the first part of the theorem. 
    For the second part of theorem, we have 
    \begin{equation*}
    \frac{\partial q}{\partial \beta}=\frac{p (c-1)(1-2p)}{(cp\beta+1-2p)^2}.
    \end{equation*}
    As $p>1/2$, $q$ is a decreasing function of $\beta$. As $\beta\ge p$, $q$ is maximized when $\beta=p$. In this case, we can obtain $\gamma=p$ by solving $p\gamma=p\beta+\gamma-p$. $\beta=\gamma=p$ shows that $\mathbb{P}[r^2_v=l(v)]$ does not depend on $r^1_v$, which means that the two models are independent.    
\end{proof}

\section{Details of Experimental Evaluation}
\label{sec:exp_details}

\subsection{Implementation Details}
\label{sec:imp_details}

\paragraph{Evaluation protocol.} To eliminate the influence of random factors and ensure that the performance comparison is fair, we adopted the evaluation protocol provided by \cite{shchur2019pitfalls}. A 20/30/rest split for train/val/test set was used for all the datasets. In the experiments, we evaluated each model on 10 randomly generated dataset splits, and under each split, we ran the model for 10 times using different random seeds. We reported the mean value and standard deviation of the test accuracies  across the 100 runs for each model on each dataset. For the experiments comparing Self-Enhanced GNN with the base GNN models (i.e., GCN, GAT and SGC), all model implementation and evaluation settings were kept fixed and identical.

%All random seeds and dataset splits are fixed for the .           

%In order to eliminate random effects and make the evaluation procedures as robust as possible, we adopted the evaluation procedures proposed by \citet{shchur2019pitfalls}. According to \citet{shchur2019pitfalls}, in semi-supervised node classification task, we use random 20/30/rest split as train/val/test split for all datasets. In the evaluation, we evaluate each model on 10 different random generated splits, and for each splits we run 10 times using different random seeds. We report the mean value and standard deviation of the test accuracies of the total 100 runs for each model. All the random seeds and random datasets splits are fixed before evaluation.

%For both experiments with or without Self-Enhanced GNN, all model implementations and evaluation frameworks settings are fixed and identical.

\paragraph{Structure of the base models}. Our GCN model implementation has 2 GCN convolutional layers with a hidden size of 16. The activation function is \textit{ReLU}. A dropout layer with a dropout rate of 0.5 is used after the first GCN layer. Our GAT model implementation has 2 GAT layers with an attention coefficient dropout probability of 0.6. The first layer is an 8-heads attention layer with a hidden size of 8. The second layer has a hidden size of $8 \times 8$. The activation function is \textit{ELU} . Two dropout layers with a dropout rate of 0.6 are used between the input layer and the first GAT layer, and between the first GAT layer and the second GAT layer. Our SGC model implementation has an SGC convolutional layer with 2 hops (equivalent to 2 SGC layers according to the SGC definition).  

%For GCN model implementation, it has 2 GCN convolutional layers with hidden size of 16. It use \textit{ReLU} as the activation function. A dropout layer with a dropout rate of 0.5 is used after the first GCN layer.

\paragraph{Model training.} We used the Adam optimizer~\citep{kingma2014adam} with a learning rate of 0.01 and an $L_2$ regularization coefficient of $5e^{-4}$. We did not use learning rate decay and early stopping. As the difficulty of model training varies for different datasets, we used a different number of training epochs for each dataset, i.e., 400 epochs for \cora, \citeseer, \pubmed and \magph, 1000 epochs for \amcomp, 2000 epochs for \amphoto and \magcs. 

\paragraph{Software.} All the models and algorithms in the experiments are implemented on PyTorch~\cite{pytorch} and PyTorch-Geometric~\cite{pytorch-geometric}. The software versions are python=3.6.9, torch=1.2.0, CUDA=10.2.89, pytorch\_geometric=1.3.2.

\paragraph{Topology update.} For \textit{Delete}, before edge deletion, we remove all self-loop edges in the original graph. Then the edges are deleted according to Algorithm~\ref{alg:edge deletion} with a threshold. After edge deletion, we add back the removed self-loop edges. For \textit{Add}, we constrain the number of added edges to be less than 4 times of the number of edges in the original graph. This threshold is used to decide the number of candidate edges for addition, i.e., $k$. We get the top-$k$ edges from the  $n\times n$ potential edges according to the label correlation (i.e., $g_v^{\top}g_u$). After filtering the edges already in the graph, we add new edges using Algorithm~\ref{alg:edge addition}. For \textit{Modify}, we constrain the total number of added edges to be the same as the number of deleted edges because tuning the parameters for edge deletion and addition jointly will result in high complexity. This constraint also helps maintain the graph topology to some degree by not changing the structure too much. We conduct edge deletion first, and then add the same number of edges as that of the deleted edges. We ensure that deleted edges will not be added back. 

%as there are $n\times n$ possible edges (with $n$ being the number of nodes in the graph),

\paragraph{Training node augmentation.} For training node augmentation,  we use two models trained with swapped training and validation set to label the nodes in the test set. Only the nodes having the same label prediction from the two models can be added to the  augmented training set. A confidence threshold is used to control the number of pre-selected nodes for addition. We count the number of nodes from each class in the pre-selected nodes and obtain the class with the minimum number of pre-selected nodes. This number is used to control the number of added nodes for all classes (i.e., the class balance trick) to avoid introducing additional biases. 

\paragraph{Joint use of TU and TNA.} For experiments that jointly use topology update and training node augmentation, we apply the two techniques independently and use the thresholds selected by each algorithm individually to avoid the high complexity of joint parameter tuning. Denote the optimal parameter for topology update and training node augmentation as $\tau_{tu}$ and $\tau_{tna}$, respectively. We consider three configurations, i.e., $(\tau_{tu}, 0)$, $(0, \tau_{tna})$ and $( \tau_{tu}, \tau_{tna})$ (setting the $\tau=0$ means disabling the algorithm) and select the best configuration using the validation accuracy. The reported results is the test accuracy of the selected configuration. Therefore, our framework still has the potential to perform even better if more fine-grained tuning on the thresholds parameters are conducted.

All the thresholds mentioned above are determined totally by the classification accuracy on the validation set.

\end{document}